\documentclass{article}

\usepackage{microtype}
\usepackage{graphicx}
\usepackage{subfigure}
\usepackage{booktabs} % for professional tables

\usepackage{hyperref}

\usepackage[accepted]{icml2024}

\usepackage{amsmath}
\usepackage{amssymb}
\usepackage{mathtools}
\usepackage{amsthm}

\usepackage[capitalize,noabbrev]{cleveref}

\usepackage{amsfonts}
\usepackage{nicefrac}
\usepackage[T1]{fontenc}  
\usepackage{xcolor}

\usepackage{caption}
\usepackage{subcaption}

\theoremstyle{definition}
\newtheorem{definition}{Definition}[section]
\newtheorem{theorem}[definition]{Theorem}
\newtheorem{corollary}[definition]{Corollary}
\newtheorem{proposition}[definition]{Proposition}
\newtheorem{lemma}[definition]{Lemma}
\newtheorem*{remark}{Remark}

\theoremstyle{remark}

\DeclareMathOperator{\cbias}{c-bias}
\DeclareMathOperator{\bias}{bias}
\DeclareMathOperator{\pbias}{pos-bias}
\DeclareMathOperator{\nbias}{neg-bias}
\DeclareMathOperator{\PR}{PR}

\DeclareMathOperator{\TPR}{TPR}
\DeclareMathOperator{\FPR}{FPR}
\DeclareMathOperator{\ROC}{ROC}
\DeclareMathOperator{\gini}{Gini}

\DeclareMathOperator{\AUROC}{AUROC}
\DeclareMathOperator{\cond}{\,|\,}%conditional on
\usepackage{blindtext}
\newcommand{\indep}{\perp \!\!\! \perp}

\usepackage[textsize=tiny]{todonotes}

\icmltitlerunning{Standardized Interpretable Fairness Measures for Continuous Risk Scores}

\begin{document}

\twocolumn[
\icmltitle{Standardized Interpretable Fairness Measures for Continuous Risk Scores}



\icmlsetsymbol{equal}{*} %todo

\begin{icmlauthorlist}
\icmlauthor{Ann-Kristin Becker}{comp}
\icmlauthor{Oana Dumitrasc}{comp}
\icmlauthor{Klaus Broelemann}{comp}
%\icmlauthor{Firstname4 Lastname4}{sch}
%\icmlauthor{Firstname5 Lastname5}{yyy}
%\icmlauthor{Firstname6 Lastname6}{sch,yyy,comp}
%\icmlauthor{Firstname7 Lastname7}{comp}
%\icmlauthor{Firstname8 Lastname8}{sch}
%\icmlauthor{Firstname8 Lastname8}{yyy,comp}
\end{icmlauthorlist}

%\icmlaffiliation{yyy}{Department of XXX, University of YYY, Location, Country}
\icmlaffiliation{comp}{SCHUFA Holding AG, Wiesbaden, Germany}
%\icmlaffiliation{sch}{School of ZZZ, Institute of WWW, Location, Country}

%\icmlcorrespondingauthor{Firstname1 Lastname1}{first1.last1@xxx.edu}
%\icmlcorrespondingauthor{Firstname2 Lastname2}{first2.last2@www.uk}

\icmlcorrespondingauthor{Ann-Kristin Becker}{ann-kristin.becker@schufa.de}
\icmlcorrespondingauthor{Oana Dumitrasc}{oana.dumitrasc@schufa.de}
\icmlcorrespondingauthor{Klaus Broelemann}{klaus.broelemann@schufa.de}

\icmlkeywords{Machine Learning, ICML}

\vskip 0.3in
]



\printAffiliationsAndNotice{}  % leave blank if no need to mention equal contribution
%todo

\begin{abstract}
We propose a standardized version of fairness measures for continuous scores with a reasonable interpretation based on the Wasserstein distance. Our measures are easily computable and well suited for quantifying and interpreting the strength of group disparities as well as for comparing biases across different models, datasets, or time points. We derive a link between the different families of existing fairness measures for scores and show that the proposed standardized fairness measures outperform ROC-based fairness measures because they are more explicit and can quantify significant biases that ROC-based fairness measures miss. 
\end{abstract}

\section{Introduction}
In recent years, many decision-making processes in areas such as finance, education, social media or medicine %\todo[author=Oana]{In Zotero tags: Reference fairness paper, Recidivism / Health / Social Media / Education}
have been automated, often at least in part with the goal of making those decisions more comparable, objective, and non-discriminatory \cite{estevaDermatologistlevelClassificationSkin2017, holsteinStudentLearningBenefits2018, alvaradoAlgorithmicExperienceInitial2018, bucherAlgorithmicImaginaryExploring2017, raderUnderstandingUserBeliefs2015}.  
For high-risk business transactions between individuals and companies (e.g. in the lending industry), often predictions of machine learning algorithms are incorporated into those decisions. Such algorithms aim to differentiate individuals as optimally as possible based on historical data and in terms of future behavior. They assign risk scores or risk categories to individuals. 
Even with good intentions, the approach runs the risk of directly or indirectly discriminating against individuals on the basis of protected characteristics, such as gender, ethnicity, political background or sexual orientation \cite{larson2016we, datta2014automated, Kchling2020DiscriminatedBA}. That may be the case, if the data reflects biased social circumstances or include prejudicial historical decisions.

Such discriminatory predictions manifest as disparities among protected groups and may occur in different forms and for various reasons. For example, individuals belonging to different protected groups may be assigned different scores even if they have the same outcome, or predictions may turn out to have different levels of consistency with the ground-truth risk. Unfortunately, in most cases different notions of algorithmic fairness are incompatible \cite{barocasFairnessMachineLearning2019, kleinbergInherentTradeOffsFair2018, saravanakumarImpossibilityTheoremMachine2021, chouldechovaFairPredictionDisparate2017b, pleiss2017fairness}. Various measures for algorithmic fairness have been developed that aim to quantify different kinds of group disparities \cite{zafarFairnessDisparateTreatment2017, kamishimaFairnessAwareClassifierPrejudice2012, makhloufApplicabilityMachineLearning2021}. 
So far, most of the available literature discusses the problem in the context of binary decision tasks \cite{mitchellAlgorithmicFairnessChoices2021, barocasFairnessMachineLearning2019, kozodoi2022fairness}. 

However, in many applications, neither a final decision is known, nor is the explicit cost of false predictions. 
This is especially be the case when score and decision are performed by different entities. A prominent example is the COMPAS Score \cite{larson2016we} which was developed by one entity to support decisions done by other entities.
It may also be that a score is never applied as a pure decision but only as a quantitative prediction that affects, e.g. the cost of a product (risk-based pricing). In these cases, fairness can only be fully assessed if the disparities between groups are summarized across the entire score model. 

% Since regularly measuring group disparities of score models can help increase the visibility of group effects that might otherwise go unnoticed - especially in the case of underrepresented or historically disadvantaged groups - we need flexible and transparent measures that correctly quantify group disparities for continuous scores and allow us to interpret their implications. 
% This paper presents a new approach to quantifying group disparities for continuous score models in a well interpretable and mathematically sound manner. We present a distribution-invariant framework, that allows for monitoring bias over time or between models and populations, even if there is a shift in the score distribution. Moreover, the setting allows to bridge the gap between common fairness-metrics stemming directly from three parity concepts \cite{kleinbergInherentTradeOffsFair2018, hardtEqualityOpportunitySupervised2016, barocasFairnessMachineLearning2019, makhloufApplicabilityMachineLearning2021} and ROC-based approaches \cite{vogelLearningFairScoring2021, kallusFairnessRiskScores2019, yang2022minimax, beutelFairnessRecommendationRanking2019}. 

This paper presents a novel approach to quantifying group disparities for continuous risk score models. It's major contributions are
\begin{itemize}
    \item a well interpretable and mathematically sound method for quantifying group disparities in continuous risk score models.
    \item a standardized framework, that allows for monitoring bias over time or between models and populations, even if there is a shift in the score distribution. 
    Furthermore, standardized measures are unaffected by monotonic transformations of the scores, such as logistic / logit transform. This prevents malicious actors from finding a transformation that hides the bias (see section~\ref{sec:standardized-measures}).
    \item bridging the gap between common fairness-metrics stemming directly from three parity concepts \cite{kleinbergInherentTradeOffsFair2018, hardtEqualityOpportunitySupervised2016, barocasFairnessMachineLearning2019, makhloufApplicabilityMachineLearning2021} and ROC-based approaches \cite{vogelLearningFairScoring2021, kallusFairnessRiskScores2019, yang2022minimax, beutelFairnessRecommendationRanking2019}.
\end{itemize}

As not all group disparities arise from discriminatory circumstances - even large disparities between groups may be explainable or justifiable otherwise - assessing whether disparities are unfair should entail a more detailed analysis of their underlying causes and drivers. Thus, to be explicit, we use the term \emph{disparity measure} instead of \emph{fairness measure} throughout the rest of the paper to underline that all discussed measures are purely observational.

The paper is structured as follows: Most of the available quantitative disparity metrics for classifiers reduce down to three main parity concepts that are based on conditional independence: Independence, separation and sufficiency \cite{barocasFairnessMachineLearning2019, makhloufApplicabilityMachineLearning2021, kozodoi2022fairness}. In Section \ref{section:classifiers}, we discuss these concepts and existing related measures in terms of binary classifiers first, and generalize them to continuous risk scores in Section \ref{section:scores}. We show that our proposed measures are more flexible than many existing metrics and we discuss their interpretability. In Section \ref{section:roc}, we compare the presented measures to ROC-based disparity measures, and we prove that our proposed measures impose a stronger objective and are better suited to detect bias. We outline published related work throughout each section. Section \ref{section:sims} contains results of experiments using benchmark data and Section \ref{section:outlook} includes final discussion and outlook. All proofs of technical results are deferred to the appendix. 
\section{Parity concepts and fairness measures for classifiers}
\label{section:classifiers}
Let $Y$ denote a binary target variable with favorable outcome class $Y=0$ and unfavorable class $Y=1$, and $X$ a set of predictors. Let $S \in \mathcal{S} \subset \mathbb{R}$ denote an estimate of the posterior probability of the favorable outcome of $Y$, $\mathbb{P}(Y=0\,|\,X)$ or some increasing function of this quantity, %in SCHUFA case of P(Y=0)
in the following called \emph{(risk) score}, with cumulative distribution function $F_S$ and density function $f_S$. We assume $\mathcal{S}$ to be bounded with $|\mathcal{S}|=\sup \mathcal{S}-\inf \mathcal{S}$ denoting the length of the score range. Let $A$ be a (protected) attribute of interest defining two (protected) groups ($A \in \{a,b\}$ binary w.l.o.g.). We choose $A=b$ as the group of interest, e.g. the expected discriminated group. All discussed measures are purely observational and based on the joint distribution of $(S,A,Y).$ They can be easily calculated if a random sample of the joint distribution is available.

Note that each continuous score $S$ induces an infinite set of binary classifiers by choosing a threshold $s \in \mathcal{S}$ and accepting every sample with $S>s$. We define disparity measures for binary classifiers in dependence of such a threshold value $s$.
For a group $A$, the \emph{positive rate} at a threshold $s$ is given by $\PR_{A}(s)=\mathbb{P}(S> s|A)=1-F_{S|A}(s),$ the \emph{true positive} and \emph{false positive rates} by $\TPR_{A}(s)=1-F_{S|A,Y=0}(s)$ and $\FPR_{A}(s)=1-F_{S|A,Y=1}(s)$, respectively.
We will write in short $F:=F_S$ and $f:=f_S$, as well as $S_{ay}:= S|A=a, Y=y$, $F_{ay}:=F_{S|A=a, Y=y}$ and $S_{by}, F_{by}, f_{ay}, f_{by}$ for the conditional random variables, distribution functions and density functions. 
For a cumulative distribution function $G$, we denote by $G^{-1}$ the related quantile function (generalized inverse) with $G^{-1}(p)=\inf\{x \in \mathbb{R}: p \leq G(x)\}$ which fulfills $G^{-1}(G(X))=X$ almost surely. If $G$ is continuous and strictly monotonically increasing, then the quantile function is the inverse.
\paragraph{Independence (selection rate parity)}
The random variables $S$ and $A$ satisfy independence if $S\indep A,$
which implies $F_{S|A=a} = F_{S|A=b} = F_S.$
Group disparity of classifiers can be quantified by the difference between the positive rates \cite{makhloufApplicabilityMachineLearning2021, zafarFairnessDisparateTreatment2017, dworkFairnessAwareness2012}
\begin{equation}
\begin{split}
\cbias_{\text{IND}}(S_a, S_b; s) &=\PR_b(s) - \PR_a(s) \\
 &= F_{a}(s)-F_{b}(s) .
\end{split}
\end{equation}

The concept of independence contradicts optimality $S=Y$, if $Y \not\!\indep A$ and is, thus, not an intuitive fairness measure in most cases. 
On the other hand, the following two measures, separation and sufficiency, are both compatible with optimality and allow $A\not\!\indep Y$, as they include the target variable $Y$ in the independence statements and allow for disparities that can be explained by group differences in the ground-truth.
\paragraph{Separation (error rate parity)}
The random variables $S$, $A$ and $Y$ satisfy separation if 
${S \indep A \cond Y}.$
For a binary outcome $Y$, the separation condition splits into \emph{true positive rate parity} $F_{S|A=a,Y=0} = F_{S|A=b,Y=0} = F_{S|Y=0}$ (\emph{equal opportunity}, EO) \cite{zhangEqualityOpportunityClassification2018, hardtEqualityOpportunitySupervised2016} and \emph{false positive rate parity} $F_{S|A=a,Y=1} = F_{S|A=b,Y=1} = F_{S|Y=1}$ (\emph{predictive equality}, PE) \cite{corbett-daviesAlgorithmicDecisionMaking2017, makhloufApplicabilityMachineLearning2021}. If both hold, the condition is also known as \emph{equalized odds} \cite{makhloufApplicabilityMachineLearning2021, hardtEqualityOpportunitySupervised2016}.
Group disparity of classifiers can be quantified by the difference between the true and false positive rates
\begin{equation}
\begin{split}
\cbias_{\text{EO}}(S_a, S_b; s) & = \TPR_ b(s)-\TPR_a(s) \\ & = F_{a0}(s)-F_{b0}(s),
\end{split} 
\end{equation} 
\begin{equation}
\begin{split}
\cbias_{\text{PE}}(S_a, S_b; s) &= \FPR_b(s)-\FPR_a(s) \\ &= F_{a1}(s)-F_{b1}(s). 
\end{split} 
\end{equation} 

\paragraph{Sufficiency (predictive value parity)}
The random variables $S$, $A$ and $Y$ satisfy sufficiency if 
$ Y \indep A \cond S$ (in words, $S$ is sufficient to optimally predict $Y$). %(additionally knowing $A$ does not bring any advantage).
Sufficiency implies group parity of positive and negative predictive values. However, especially in case of continuous scores, usually, \emph{calibration} within each group \cite{kleinbergInherentTradeOffsFair2018} (resp. \emph{test fairness} \cite{chouldechovaFairPredictionDisparate2017b}), as an equivalent concept, is used instead \cite{barocasFairnessMachineLearning2019}.
The calibration bias examines if the model's predicted probability deviates similarly strongly from the true outcome rates within each group:
\begin{equation}
    \begin{split}
    \cbias_{\text{CALI}}(S_a, S_b; s)&=\mathbb{P}(Y=0|A=b,S=s) \\& - \mathbb{P}(Y=0|A=a,S=s).
\end{split}
\end{equation}

\emph{Well-calibration} \cite{kleinbergInherentTradeOffsFair2018, pleiss2017fairness} additionally requires the prediction of both groups to accurately reflect the ground truth $\mathbb{P}(Y=0|A,S=s)=s$. 
For determining the calibration difference, the score range is usually binned into a fixed number of intervals. A high calibration bias reflects the fact that (for a given score $s$) the lower-risk group carries the costs of the higher-risk group.
The concept of sufficiency is especially important if the model is applied in a context, where both, the score and the group membership are available to the decision maker. Then, a high calibration bias will evoke a group-specific interpretation and handling of identical score values. 
On the other hand, sufficiency does not prevent discrimination: high- and low-risk individuals of a group can be mixed and assigned an intermediate risk score without violating sufficiency. Moreover, sufficiency is often naturally fulfilled as a consequence of unconstrained supervised learning, especially if the group membership is (at least to some extent) encoded in the input data. Thus, it is usually not a constraint and not a trade-off with predictive performance \cite{liuImplicitFairnessCriterion2019}. 

If separation is violated, the model output includes more information about the group $A$ as is justified by the ground truth $Y$ alone. So, different groups carry different costs of misclassification. It is therefore a reasonable concept for surfacing potential inequities. Conversely, a violation of sufficiency results in a different calibration and a different meaning of identical score values per group. That is the case, if the relation of $A$ and $Y$ is not properly modeled by the score.

In general, independence, separation and sufficiency are opposing concepts. It can be shown that for a given dataset, except for special cases (like perfect prediction or equal base rates), every pair of the three parity concepts is mathematically incompatible \cite{barocasFairnessMachineLearning2019, kleinbergInherentTradeOffsFair2018, saravanakumarImpossibilityTheoremMachine2021,chouldechovaFairPredictionDisparate2017b, pleiss2017fairness}.

\section{Generalization to continuous risk scores}
\label{section:scores}
We propose to use the expected absolute classifier bias as a disparity measure for scores. Note, that an expected value of zero implies that every classifier derived from the score by choosing a group-unspecific threshold will be bias-free. By evaluating and aggregating the bias across all possible decision thresholds, this generalization serves as a useful diagnostic tool in fairness analyses and follows a similar idea as used in ROC analyses. The two proposed versions can be seen as generalized rate differences. They differ only in the way, in which possible thresholds are weighted. We show, that for the concepts independence and separation, the proposed disparity measures are identical to Wasserstein distances between the groupwise score-distributions.

% \subsection{Desiderata of Disparity Measures for Scores}
% \begin{enumerate}
%     \item \textbf{Interpretability}: There is a way to interpret the strength of the bias and compare it between different models and / or different populations (apart of an interpretation in terms of statistical significance, which is highly dependent on the sample size).
%     \item \textbf{Flexibility / Continuity}: The measure can detect different forms of disparity, is zero only in case of parity of the score distributions and it changes continuously if the distance between groups grows.
%     \item \textbf{Population-Independence}: The measure is independent of the sample size.
% \end{enumerate}

The use of Wasserstein distance in previous works has focused mainly on independence fairness (e.g. demographic parity), therefore a consideration of all three disparity concepts (independence, separation and sufficiency / calibration) for continuous risk scores is novel to this work.

% \subsection{Maximal Classifier Bias}
\subsection{Expected classifier bias with uniformly weighted thresholds}
\begin{definition}
By assuming each threshold $s\in \mathcal{S}$ is equally important, we define
\begin{equation}
 \begin{split}
    \bias_x^{\mathcal{U}}(S_a, S_b)&:= \mathbb{E}_{S\sim\mathcal{U}}[|\cbias_x(S_a, S_b; S)|] \\
    &= \frac{1}{|\mathcal{S}|}\int_{\mathcal{S}}|\cbias_x(S_a, S_b; s)| \, ds.
 \end{split}
\end{equation}

\end{definition}

\begin{theorem}
\label{uniformbias}
For the concepts independence and separation, i.e for $x \in$ \{IND, PE, EO\}, it holds:
\begin{enumerate} 
    \item[(i)] $\bias_x^{\mathcal{U}}(S|A=a, S|A=b)$ is equal to the normalized Wasserstein-1-distance between the conditional score distributions in the groups over the (finite) score region $\mathcal{S}$ i.e. 
\begin{equation}
    \bias_x^{\mathcal{U}}(S_a, S_b)=\frac{1}{|\mathcal{S}|}\cdot W_1(S_{ay}, S_{by}),
\end{equation}
where $y=0$ for $x=\text{EO}$, $y=1$ for $x=\text{PE}$, and $y=\cdot$ for $x=\text{IND}$.
\item[(ii)] As a consequence, we can derive the disparity between average scores per group (known as \emph{balance for the positive / negative class} \cite{kleinbergInherentTradeOffsFair2018}) as a lower bound, i.e.
\begin{align}
\bias_x^{\mathcal{U}}(S_a, S_b)\geq \frac{1}{|\mathcal{S}|} \left|\mathbb{E}[S_{by}]-\mathbb{E}[S_{ay}]\right| .
\end{align}
\end{enumerate}
\end{theorem}

A similar version of Theorem~\ref{uniformbias} (i) for independence bias has previously be presented by \citet{jiang20a}, but they did not draw the connection to the \emph{balance for the positive / negative class}. 
%The Wasserstein distance was proposed as a fairness measure recently by \citet{miroshnikovWassersteinbasedFairnessInterpretability2022} or \citet{kwegyir-aggreyEverythingRelativeUnderstanding2021}, and it was especially used for debiasing purposes earlier \cite{miroshnikovModelagnosticBiasMitigation2021, hanRetiringDeltaDP2023, chzhenFairRegressionWasserstein2020}.
The Wasserstein distance was recently proposed as a fairness measure  \cite{miroshnikovWassersteinbasedFairnessInterpretability2022, kwegyir-aggreyEverythingRelativeUnderstanding2021,zhao2023costs} mainly for independence bias, and it was especially used for debiasing purposes earlier \cite{miroshnikovModelagnosticBiasMitigation2021, hanRetiringDeltaDP2023, chzhenFairRegressionWasserstein2020}.
Fairness of scores has also been subject for regression tasks~\cite{Agarwal2019FairRQ,pmlr-v206-wei23a,zhao2023costs}. Again, due to the different target value, only for independence bias.
A formal definition and properties of $W_1$ can be found in the appendix. For calibration, the bias $\bias_\text{CALI}^\mathcal{U}$ is equal to the two-sample version of the 
the $l_1$-calibration error \cite{kumarVerifiedUncertaintyCalibration2019}.
\subsection{Standardized Measures}
\label{sec:standardized-measures}
It can be difficult to compare the expected classifier bias of datasets with distinct score distributions. Especially for imbalanced datasets score distributions are often highly skewed. In this case, disparities in dense score areas may be more critical as they affect more samples. Therefore, we developed a method that standardizes the bias computation, making it independent of data skewness.

% start #camera#ready
Our standardized disparity measures for risk scores are important especially when a monotonic transformation is applied to the score. A good example of such a scenario is given by the logistic regression, where both the probability or the linear term can be used as a score. The risk assessment of both variants is the same. It is also most likely that down-stream tasks would adopt to the score representation (linear term /probability) used. This means, both representations are likely to lead to the same treatment in down-stream tasks and to the same (un)fairness. Without invariance to monotonic transformations, the two representations would have different bias-measures.
% end #camera#ready

In a worst-case scenario an entity could apply a strictly monotonic function to their score, stretching areas with low bias and shrinking areas with high bias. Doing so would allow to mask the bias without any change in accuracy or better ranking of the disadvantaged group. This has already be proposed~\cite{jiang20a}.

That is why we propose an alternative generalization that weights the thresholds by their frequency observed in the population. By this, the resulting disparity measures become independent of the concrete distribution and evaluate the fairness of a bipartite ranking task, similar to ROC measures. Each sample is equally important in this scenario.

% start #camera#ready
% Our measures are purely rank-based and in this sense they are independent of the concrete score distributions. 
As a consequence, this allows for a meaningful comparison of different scores, even of scores with different ranges (e.g. a normally-distributed score that can take any real value and a uniformly-distributed score that only takes probabilities). 
Our methodology can thus be utilized to assess the effectiveness of debiasing approaches~\cite{hort2023bias}.

%Our metrics only compare bias in the rankings of the (protected) groups, independent of any further distributional information. 

% This means that our metric is invariant of any distribution, covariance or label shift that does not affect the ranking of the groups. Unfortunately, an in-depth analysis when the ranking is unchanged is out of scope for this work. But note that Lemma B.3 shows that our method is independent of the ratio a group has in the total population.
% end #camera#ready
\begin{definition}

\begin{align}
    \bias_x^{S}(S_a, S_b)&:=\mathbb{E}_{S\sim F}[|\cbias_x(S_a, S_b; S)|] \\
    &= \int_0^1 |\cbias_x(S_a, S_b; F^{-1}(r))|\, dr  \nonumber \\ % can be shown by substitution with f=c-bias(F^-1(x)) and g=F(x)
    &= \int_\mathcal{S} |\cbias_x(S_a, S_b; s)| \cdot f(s)\, ds \label{density-wasserstein}
\end{align}

\label{definition-bias-S}
\end{definition}
Note that $\bias_x^{S}(S|A=a, S|A=b)$ is invariant under monotonic score transformations as it is a purely ranking-based metric, $\bias_x^{\mathcal{U}}(S|A=a, S|A=b)$ is not. If $S \sim \mathcal{U}$ it holds $\bias^S=\bias^{\mathcal{U}}$. 
We show, that the standardized bias is equal to the Wasserstein-1-distance between quantile-transformed distributions. To our knowledge, this is the first introduction of a fairness measure based on the Wasserstein distance, which is invariant to transformations. 
\begin{theorem}
\label{quantilebias}
For the concepts independence and separation, i.e. for $x \in$ \{IND, PE, EO\}, it holds:
\begin{enumerate} 
    \item[(i)] $\bias_x^{S}$ is equal to the Wasserstein-1-distance using the push-forward by the quantile function $F^{-1}\#\mathcal{L}_1$ as ground metric (with $y=0$ for $x=\text{EO}$, $y=1$ for $x=\text{PE}$, and $y=\cdot$ for $x=\text{IND}$)
    \begin{equation}
    \begin{split}
        &\bias_x^{S}(S_a, S_b)=W_1(F(S_{ay}),F(S_{by})) \\ &= \int_0^1 |F_{ay}\circ F^{-1}(t)-F_{by}\circ F^{-1}(t)|\,dt.
    \end{split}
    \end{equation}
    
    \item[(ii)] We can derive the disparity between the average relative rank per group as a lower bound. 
    % \item[(iii)] If ???\todo[author=Ann-Kristin]{under which assumptions?}, $\bias_x^{S}$ takes values between $0$ and $0.5$. If the sample consists of only two individuals from different groups with different scores, $\bias_x^{S}=1$. The same holds for larger groups, if the score distributions are singular. % If the scores are not identical, for balanced groups, the maximal bias is 1/2+1/n (5+5: 0.6, 10+10: 0.55,...). Relation to entropy???
\end{enumerate}
\end{theorem}
For reasons of simplicity, we will use the notation $W_Z(X, Y) := W_1(F_Z(X), F_Z(Y)).$

\subsection{Interpretation of the score bias}
In general, $\bias^\mathcal{U}$ and $\bias^{S}$ take values in the interval $[0,1]$ as they are expected values over rate differences. The optimal value, a bias of zero, indicates group parity for all decision thresholds with respect to the analyzed type of classifier error. When comparing multiple score models or one model over multiple populations, a smaller bias is preferable. 
The standardized method allows direct comparison of models with different score distributions with respect to group parity in bipartite ranking tasks. 
$\bias^\mathcal{U}$ and $\bias^{S}$ can be interpreted as the classifier bias to be expected at a randomly chosen threshold - either randomly selected from all available score values ($\bias^\mathcal{U}$) or by randomly selecting one sample and assigning the favorable label to all samples that are ranked higher ($\bias^{S}$).

In addition, the separation and independence biases can be interpreted in terms of the Wasserstein distance (or \emph{Earth Mover distance}): The bias is measured as the minimum cost of aligning the two groups with respect to the analyzed type of classifier error. Here, the baseline distance is measured in normalized scores for $\bias^\mathcal{U}$ or in ranks for $\bias^{S}$. It indicates what proportion of a group must be scored (how) differently in order to equalize the groups.

\subsection{Positive and negative components of the score bias}
\label{sec:pos-neg-components}
Unlike a classifier bias, a score bias does not have to be overall positive or negative for a particular group. Instead, there may be thresholds at which one group is disadvantaged and others at which the opposing group is disadvantaged. To further analyze the bias, we can decompose the total bias into a positive and a negative component (positive and negative from the point of view of the chosen disadvantaged group, here $b$).
For this purpose, the classifier bias is divided into a positive and a negative part for each threshold
\begin{align*}
    \cbias^+(s) &= \max(\cbias(s), 0)  \quad \textrm{and} \\
    \cbias^-(s) &= -\min(\cbias(s), 0).
\end{align*} 
This allows to derive a decomposition of both score bias types into two components:  
\begin{align}
    \pbias_x(S_a, S_b) &=  \mathbb{E}[\cbias_x^+(S_a, S_b; S)], \\
    \nbias_x(S_a, S_b) &= \mathbb{E}[\cbias_x^-(S_a, S_b; S)],
\end{align}
where $\bias_x(S_a, S_b) = \pbias_x(S_a, S_b) + \nbias_x(S_a, S_b).$
By dividing each component by the total bias, a percentage can be calculated. The decomposition helps to interpret, which of the two compared groups is affected predominantly negatively by the observed bias. A similar decomposition of a Wasserstein bias was proposed by \citet{miroshnikovWassersteinbasedFairnessInterpretability2022}. 

\section{ROC-based fairness measures and relations}
\label{section:roc}
Furthermore, there exists a wide variety of (separation) fairness metrics which are calculated based on ROC curves or the area under the curves. We show, that the proposed standardized bias measures outperform these ROC-based measures as they are more explicit, easier to interpret, and can measure biases, that ROC-based fairness measures cannot catch. 
We define the ROC curve between two arbitrary random variables $G, H$, similar to \citet{vogelLearningFairScoring2021}.
In a bipartite ranking or scoring task, the ROC curve is usually used to evaluate the separability between positive and negative outcome class. In this case, $G=S_0, H=S_1$.

\begin{definition}[ROC]
Let $G$ and $H$ be two random variables with cumulative distribution functions $F_G, F_H$ on $\mathbb{R}$ with quantile functions $F_G^{-1}, F_H^{-1}$. Then the ROC curve of $G$ and $H$ is the mapping 
\begin{align}
\ROC_{G,H}: p \in [0,1] \mapsto  1-F_G(F_H^{-1}(1-p))
\end{align}

%$$ \ROC_{S_0,S_1}(p) = \TPR_S(\FPR_S^{-1}(p)) =  1-F_{Y=0}(F_{Y=1}^{-1}(1-p)),$$
%\todo[author=Ann-Kristin]{simple notation needed, that works also for two groups a,b with varying order but I don't like to read the small indices}

with the area under the curve (AUROC) and the Gini coefficient defined as
\begin{equation}
    \begin{split}
  \AUROC(G,H) &= \int_0^1 \ROC_{G,H}(p) \, dp\quad\text{and}\quad \\ \gini(G,H) &= 2 \cdot \AUROC(G,H) - 1.
%= \int_0^1 1-F_G(F_H^{-1}(1-p)) \, dp   
%= \int_0^1 F_G(F_H^{-1}(p)) \, dp - 1  
    \end{split}
\end{equation}

\end{definition}

\begin{definition}
    Similar to the above introduced biases, a ROC-based disparity-measure for score models can be defined as the expected absolute difference between two ROC curves
\begin{equation}
        \begin{split}
            &\bias_{\text{ROC}}(S_a, S_b) = \mathbb{E}[|\ROC_{S_{b0},S_{b1}}-\ROC_{S_{a0},S_{a1}}|] \\ &= \int_0^1 |\ROC_{S_{b0},S_{b1}}(s)-\ROC_{S_{a0},S_{a1}}(s)|\,ds \nonumber
        \end{split}
    \end{equation}

\end{definition}

$\bias_{\text{ROC}}(S|A=a, S|A=b)$ is equal to the absolute between ROC area (\emph{ABROCA}) \cite{gardnerEvaluatingFairnessPredictive2019}.
%If the distribution of group $A=b$ stochastically dominates group
In general, $\bias_{\text{ROC}}(S|A=a,S|A=b) \geq |\AUROC(S_{b0}, S_{b1})-\AUROC(S_{a0}, S_{a1})|$, which is known as \emph{intra-group fairness} and often used as a fairness measure for scores \cite{vogelLearningFairScoring2021, beutelFairnessRecommendationRanking2019, borkanNuancedMetricsMeasuring2019, yang2022minimax}. If the ROC curves of two groups do not cross (i.e. one group gets uniformly better scores than the other), equality holds. As the thresholds that lead to certain ROC values (pair of FPR and TPR at a certain score threshold) are group-specific, it is not sufficient to compare intra-group ROC curves \cite{vogelLearningFairScoring2021}. Thus, we define a second ROC-based measure that compares the discriminatory power across groups and is based on the cross-ROC curve \cite{kallusFairnessRiskScores2019}.

\begin{definition}
We define the cross-ROC bias as the expected difference of the ROC curves across groups 
%$$ \xROC_{A=a,A=b}(s) = \ROC_{S_{a0},S_{b1}}(s)=\TPR_{A=a}(\FPR_{A=b}^{-1}(s))=1-F_{A=a,Y=0}(F_{A=b,Y=1}^{-1}(1-s)).$$
%The related disparity measure is given by 
\begin{equation}
    \begin{split}
    &\bias_{\text{xROC}}(S_a, S_b) = \mathbb{E}[|\ROC_{S_{b0},S_{a1}}-\ROC_{S_{a0},S_{b1}}|]  \\ &= \int_0^1 |\ROC_{S_{b0},S_{a1}}(s)-\ROC_{S_{a0},S_{b1}}(s)|\,ds \nonumber
    \end{split}
\end{equation}

\end{definition}
The cross-ROC bias evaluates the difference in separability of negatives samples in one group versus positive samples of the other group.
$\bias_{\text{xROC}}$ is always greater or equal to the related AUROC-based fairness-measure $|\AUROC(S_{a0}, S_{b1})-\AUROC(S_{b0}, S_{a1})|$, that is known as \emph{subgroup positive background negative} (BPSN) or \emph{inter-group fairness} \cite{borkanNuancedMetricsMeasuring2019, vogelLearningFairScoring2021, beutelFairnessRecommendationRanking2019, yang2022minimax}. 
% \begin{remark}
% It is possible to construct models with $\bias_{\text{ROC}} = 0$ and $\bias_{\text{xROC}} \neq 0$ and also with $\bias_{\text{xROC}} = 0$ and $\bias_{\text{ROC}} \neq 0$.
% %Probably, it is even equivalent.
% \end{remark}

\subsection{Relating Wasserstein and ROC biases}
We now reveal some connections of the standardized Wasserstein disparity measures with the ROC-based disparity measures.
We first consider the general case of the Wasserstein distance between two random variables $X,Y$ quantile-transformed by $Z$. 

For the following section, we require $\ROC_{X,X}(r) = r$. This is fulfilled, whenever $F_X$ is continuous and strictly monontonic increasing, so it permits a well-defined inverse, or if the ROC-curve is interpolated linearly from finite data. 
% If $F_{ay}$ is continuous and strictly monotonic increasing or if the ROC-curve is interpolated linearly, equality holds. Similar holds for group $b$.
% For the following section, we assume that $F_{a0}, F_{b0}, F_{a1}$ and $F_{b1}$ permit an inverse that is well-defined on $\mathcal{S}$. For cdfs of empirical samples, this can be achieved by kernel density estimation. 

\begin{theorem}
The quantile-transformed Wasserstein distance can be rewritten in terms of ROC
\label{thm:basicrocwasserstein}
\begin{equation}
    \begin{split}
         W_Z(X,Y) &= \int_0^1 |F_X(F_Z^{-1}(t)) - F_Y(F_Z^{-1}(t))| dt \\ &= \int_0^1 | \ROC_{X,Z}(t) - \ROC_{Y,Z}(t) | dt.
    \end{split}
\end{equation}

\end{theorem}

Moreover, we easily get the following result.
\begin{proposition}
    \label{prop:mixture}
    Let $Z_i$, $i=1,\ldots, n$ be random variables with values in $\mathcal{S}$ and with densities $f_i$. Let $Z_K$ be their mixture, where $K$ is a random variable with values in $\{1, \ldots, n\}$.
    Then their joint density is given by $f_{Z_K}(x) = \sum_{i=1}^n \mathbb{P}(K=i) f_i(x)$ and it holds
    \begin{equation}
        W_{Z_K}(X,Y) = \sum_{i=1}^n \mathbb{P}(K=i) W_{Z_i}(X,Y).
    \end{equation}    
\end{proposition} 
Formulating $S$ as a mixture of the two groups and two outcome classes $S_{a0}, S_{a1}, S_{b0}, S_{b1}$, we get
\begin{equation}
    \begin{split}
         W_{S}(S_{ay},S_{by}) &= w_{a0} \cdot W_{S_{a0}}(S_{ay},S_{by}) \\ &+ w_{b0} \cdot W_{S_{b0}}(S_{ay},S_{by}) \\
    &+ w_{a1} \cdot W_{S_{a1}}(S_{ay},S_{by}) \\&+ w_{b1} \cdot W_{S_{b1}}(S_{ay},S_{by}). 
    \end{split}
\end{equation}
% \begin{align}
%     W_{S}(S_{ay},S_{by}) &= w_{a0} \cdot W_{S_{a0}}(S_{ay},S_{by}) + w_{b0} \cdot W_{S_{b0}}(S_{ay},S_{by}) \\
%     &+ w_{a1} \cdot W_{S_{a1}}(S_{ay},S_{by}) + w_{b1} \cdot W_{S_{b1}}(S_{ay},S_{by}). \nonumber
% \end{align}
By looking at the different mixture components, we can reveal a connection to the ROC-based disparity measures.
\begin{lemma}
\label{lemma:wassersteinroc}
%For two groups $A=a$ and $A=b$ and fixed outcome $Y=y$, let $S_{ay}$ and $S_{by}$ be two random variables describing the score with continuous and strictly increasing cdfs $F_{ay}, F_{by}$. $\tilde{y}$
$W_{S_{ay}}(S_{ay},S_{by})$ and $W_{S_{a\tilde{y}}}(S_{ay},S_{by})$ for $\tilde{y}\neq y$ can be rewritten in terms of ROC
   % initial Version
    % \begin{align}
    %     W_{S_{ay}}(S_{ay},S_{by}) &= \int_0^1\left|\ROC_{S_{by},S_{ay}}(r)-r\right|dr, \\
    %     W_{S_{a\tilde{y}}}(S_{ay},S_{by}) &= \int_0^1 |\ROC_{S_{by}, S_{a\tilde{y}}}(r)-\ROC_{S_{ay}, S_{a\tilde{y}}}(r)| dr.
    % \end{align}
    % Klaus
    % \begin{align}
    %     &W_{S_{ay}}(S_{ay},S_{by}) \nonumber\\
    %     =&{} \int_0^1\left|\ROC_{S_{by},S_{ay}}(r)-r\right|dr, \\
    %     &W_{S_{a\tilde{y}}}(S_{ay},S_{by}) \nonumber\\
    %     =&{} \int_0^1 |\ROC_{S_{by}, S_{a\tilde{y}}}(r)-\ROC_{S_{ay}, S_{a\tilde{y}}}(r)| dr.
    % \end{align}
    % Oana
    \begin{align}
        &W_{S_{ay}}(S_{ay},S_{by}) = \int_0^1\bigg|\ROC_{S_{by},S_{ay}}(r)-r\bigg|dr, \\
          &W_{S_{a\tilde{y}}}(S_{ay},S_{by}) = \int_0^1 \bigg|\ROC_{S_{by}, S_{a\tilde{y}}}(r)- \nonumber \\ 
          &\phantom{{abcsdfghijklmnopqrst}}\ROC_{S_{ay}, S_{a\tilde{y}}}(r)\bigg| dr. 
    \end{align}
\end{lemma}

\begin{lemma}
    \label{lemma:rocwasserstein2}
     From Jensen inequality, it follows
     \begin{equation}
         \begin{split}
             &W_{S_{ay}}(S_{ay},S_{by}) \geq |\AUROC(S_{by},S_{ay}) - \tfrac{1}{2}| \\&=\tfrac{1}{2} \cdot |\gini(S_{by},S_{ay})|.
         \end{split}
     \end{equation}
    If the ROC curve does not cross the diagonal, then equality holds.
\end{lemma}

\begin{theorem}
\label{thm:sep_as_mixture} 
We can now decompose each separation bias into a sum of four ROC statements. Let $ w_{ay} = \mathbb{P}(Y=y, A=a) $ and $w_{by} = \mathbb{P}(Y=y, A=b)$, as well as $ w_y = \mathbb{P}(Y=y) $, then it holds:
\begin{equation}
    \begin{split}
       &\bias_{\text{EO}}^{S}(S_a, S_b) = w_{a0}  \int_0^1 \left|\ROC_{S_{b0},S_{a0}}(r)-r\right|dr \\&+ w_{b0}  \int_0^1 \left|\ROC_{S_{a0},S_{b0}}(r)-r\right|dr  \\
    & + w_{a1}\int_0^1 |\ROC_{S_{a0}, S_{a1}}(r)-\ROC_{S_{b0}, S_{a1}}(r)| dr \\
    & + w_{b1}\int_0^1 |\ROC_{S_{a0}, S_{b1}}(r)-\ROC_{S_{b0}, S_{b1}}(r)| dr, 
    % \\
    % \bias_{\text{PE}}^{S}(S_a, S_b)
    % &= w_{a1} \int_0^1 \left|\ROC_{S_{b1},S_{a1}}(r)-r\right|dr +w_{b1} \int_0^1 \left|\ROC_{S_{a1},S_{b1}}(r)-r\right|dr \nonumber \\
    % & + w_{a0} \int_0^1 |\ROC_{S_{a1}, S_{a0}}(r)-\ROC_{S_{b1}, S_{a0}}(r)| dr \nonumber \\
    % &+ w_{b0} \int_0^1 |\ROC_{S_{a1}, S_{b0}}(r)-\ROC_{S_{b1}, S_{b0}}(r)| dr. 
    \end{split}
\end{equation}

% \begin{align}
%     \bias_{\text{EO}}^{S}(S_a, S_b) &= w_{a0}  \int_0^1 \left|\ROC_{S_{b0},S_{a0}}(r)-r\right|dr + w_{b0}  \int_0^1 \left|\ROC_{S_{a0},S_{b0}}(r)-r\right|dr \nonumber \\
%     & + w_{a1}\int_0^1 |\ROC_{S_{a0}, S_{a1}}(r)-\ROC_{S_{b0}, S_{a1}}(r)| dr \nonumber\\
%     & + w_{b1}\int_0^1 |\ROC_{S_{a0}, S_{b1}}(r)-\ROC_{S_{b0}, S_{b1}}(r)| dr, 
%     % \\
%     % \bias_{\text{PE}}^{S}(S_a, S_b)
%     % &= w_{a1} \int_0^1 \left|\ROC_{S_{b1},S_{a1}}(r)-r\right|dr +w_{b1} \int_0^1 \left|\ROC_{S_{a1},S_{b1}}(r)-r\right|dr \nonumber \\
%     % & + w_{a0} \int_0^1 |\ROC_{S_{a1}, S_{a0}}(r)-\ROC_{S_{b1}, S_{a0}}(r)| dr \nonumber \\
%     % &+ w_{b0} \int_0^1 |\ROC_{S_{a1}, S_{b0}}(r)-\ROC_{S_{b1}, S_{b0}}(r)| dr. 
%  \end{align} 
 and analogously for $\bias_{\text{PE}}^{S}(S_a, S_b)$ by exchanging $w_{a0}$ with $w_{a1}$, $w_{b0}$ with $w_{b1}$, $S_{a0}$ with $S_{a1}$ and $S_{b0}$ with $S_{b1}$.

\end{theorem}
%\todo{Could it be useful for explanations to calculate these components?}
\begin{corollary}
From Theorem \ref{thm:sep_as_mixture} we can infer upper bounds of the separation biases and their sum

    \begin{align}
        &\bias_{\text{EO}}^{S}(S_a, S_b) \leq 1-\frac{w_0}{2} \text{ and }  \nonumber\\
        &\bias_{\text{PE}}^{S}(S_a, S_b) \leq 1-\frac{w_1}{2} \\
        &\Rightarrow \bias_{\text{EO}}^{S}(S_a, S_b) + \bias_{\text{PE}}^{S}(S_a, S_b) \leq \frac{3}{2}.
    \end{align}

% \begin{equation}
%     \begin{split}
%         &\bias_{\text{EO}}^{S}(S_a, S_b) \leq 1-\frac{w_0}{2} \text{ and } \bias_{\text{PE}}^{S}(S_a, S_b) \leq 1-\frac{w_1}{2} \\
%         &\Rightarrow \bias_{\text{EO}}^{S}(S_a, S_b) + \bias_{\text{PE}}^{S}(S_a, S_b) \leq \frac{3}{2}.
%     \end{split}
% \end{equation}
\end{corollary}

Moreover, we show that the sum of the separation biases is an upper bound (up to population-specific constants) to both ROC biases and the separability of the groups within each outcome class.

\begin{theorem}
    \label{thm:final_inequality_wasserstein_roc}
    The following inequality holds \footnote{Note, that if $F_{ay}$ and $F_{by}$ have identical supports and permit an inverse, then $\gini(S_{ay},S_{by}) = \gini(S_{by},S_{ay})$. If this symmetry is not fulfilled, the minimum of both must be used on the right side.}
   \begin{equation}
       \begin{split}
           & \quad \bias_{\text{EO}}^{S}(S_a, S_b)+ \bias_{\text{PE}}^{S}(S_a, S_b) \\ &=  W_{S}(S_{a0},S_{b0}) + W_{S}(S_{a1},S_{b1}) \\ 
         %&\geq \min(w_{a0},w_{a1}, w_{b0}, w_{b1}) \cdot (\bias_{\text{ROC}} + \bias_{\text{xROC}})  + \frac{w_0}{2} \gini(S_{a0}, S_{b0}) + \frac{w_1}{2}\gini(S_{a1}, S_{b1}) \\
         &\geq \frac{\min(w_{a0},w_{a1}, w_{b0}, w_{b1})}{2} \cdot (\bias_{\text{ROC}} + \bias_{\text{xROC}} \\&\quad +\gini(S_{a0}, S_{b0})+\gini(S_{a1}, S_{b1})).
       \end{split}
   \end{equation}
   
    % todo
    % \begin{align}
    %  & \quad \bias_{\text{EO}}^{S}(S_a, S_b)+ \bias_{\text{PE}}^{S}(S_a, S_b)  =  W_{S}(S_{a0},S_{b0}) + W_{S}(S_{a1},S_{b1}) \nonumber \\ 
    %  %&\geq \min(w_{a0},w_{a1}, w_{b0}, w_{b1}) \cdot (\bias_{\text{ROC}} + \bias_{\text{xROC}})  + \frac{w_0}{2} \gini(S_{a0}, S_{b0}) + \frac{w_1}{2}\gini(S_{a1}, S_{b1}) \\
    %  &\geq \frac{\min(w_{a0},w_{a1}, w_{b0}, w_{b1})}{2} \cdot (\bias_{\text{ROC}} + \bias_{\text{xROC}}+\gini(S_{a0}, S_{b0})+\gini(S_{a1}, S_{b1})).
    % \end{align}
\end{theorem}

Note, that the constant $\min(w_{a0},w_{a1}, w_{b0}, w_{b1})/2$ is fixed for each dataset. Thus, decreasing both separation biases leads to a decrease of the sum of both ROC biases as well as the separability of the groups within each outcome class. Especially, separation biases of zero also diminish both ROC biases.
\begin{corollary}
\label{cor:zeroseparation}
    Zero separation biases imply zero ROC biases
    \begin{equation}
        \begin{split}
            &\bias_{\text{EO}}^{S}(S_a, S_b) = \bias_{\text{PE}}^{S}(S_a, S_b) = 0 \\
            &\Rightarrow \bias_{\text{ROC}}(S_a, S_b) = \bias_{\text{xROC}}(S_a, S_b) = 0. 
        \end{split}
    \end{equation}
    The inverse does not hold. 
\end{corollary}

% \begin{example}[Example of ROC-bias zero and Wasserstein not]
% \label{xROC_example}
% Let $A \in \{0,1\}$ denote the group membership and let the groups have identical outcome class-separability and identical score distributions $$F_{S|A=a,Y=y}=F_{S|A=b,Y=y}$$ for $y=0$ and $y=1$. It then obviously holds $\ROC_{A=a} = \ROC_{A=b}$. As ROC is transformation invariant, the similarity of the ROC values remains, even if we increase the logit scores of group $A=b$ by a constant $\alpha$ $$S'_{A=b} = \sigm (\logit (S_{A=a}) + \alpha).$$ $A=a$ is now disadvantaged against group $A=b$ at every possible threshold but $\bias_{\text{ROC}}(A=a,A=b)=0$.
% \end{example}

\begin{theorem}
\label{thm:zeroseparation2}
    Moreover, if only one separation bias is zero, ROC and cross-ROC bias become equal
    \begin{equation}
        \begin{split}
            &\bias_{\text{EO}}^{S}(S_a, S_b) =0 \text{ or } \bias_{\text{PE}}^{S}(S_a, S_b) =0 \\
            \Rightarrow& \bias_{\text{ROC}}(S_a, S_b) = \bias_{\text{xROC}}(S_a, S_b).
        \end{split}
    \end{equation}
\end{theorem}

\section{Experiments} 
\label{section:sims}
We use the COMPAS dataset\footnote{https://raw.githubusercontent.com/propublica/compas-analysis/master/compas-scores-two-years.csv}, the Adult dataset\footnote{https://archive.ics.uci.edu/ml/machine-learning-databases/adult/adult.data}
 and the German Credit dataset\footnote{https://www.kaggle.com/datasets/uciml/german-credit?resource=download} to demonstrate the application of the fairness measures for continuous risk scores. For each bias, we perform permutation tests to determine statistical significance under the null hypothesis of group parity \cite{diciccioEvaluatingFairnessUsing2020, schefzikFastIdentificationDifferential2021}. 
 The core of this paper is our novel bias evaluation metric, therefore the focus of our experiments is not on achieving a low bias, but on demonstrating where and how detecting bias is useful, for example while comparing different models and analyzing debiasing approaches. 
 In addition, we perform an experiment with synthetic datasets where the equal opportunity bias is controllable by one parameter.
 % As this is a novel evaluation measure, the major goal of our experiments is not to achieve lowest bias values, but to demonstrate the use for detecting bias, especially when comparing models and analyzing debiasing approaches.
 %In addition we report the split into positive and negative components as introduced in \ref{sec:pos-neg-components}. 
 Experimental details and complete results including all presented bias types can be found in appendix. 
 %  start #camera#ready
 The code used for the experiments in this study is online available \footnote{https://github.com/schufa-innovationlab/fair-scoring}. The repository includes detailed instructions for reproducing the results.
 % end #camera#ready

\subsection{COMPAS}
We calculate the different types of biases for the famous COMPAS decile score ($n=7214$), which predicts the risk of violent recidivism within two years following release. We choose race as protected attribute and set African-America as the expected discriminated group versus Caucasian race. To be consistent with the notation in this paper, we calculate the counter-score, so that a high score stands for the favorable outcome. In contrast to the original analysis \cite{larson2016we} we calculate the bias over the entire score area. Results (Table \ref{compas-table}) show a significant separation bias against the African-American and in favor of the Caucasian race. The disadvantaged group experiences a much lower true-positive rate (rate difference in average $\bias_{\text{EO}}^S=0.16$) as well as false positive rate (rate difference in average $\bias_{\text{PE}}^S=0.15$). The calibration bias is lower and not statistically significant but predominantly in favor of the African-American race. While the ROC bias is also low (implying that the separability is equally good in both groups considered independently), the cross-ROC bias is again high. In this case, there is not much difference between $\bias^{\mathcal{U}}$ and $\bias^S$ (complete results can be found in appendix).

\begin{table}
  \caption{Bias of COMPAS score of African-American vs. Caucasian.}
  \label{compas-table}
  \vskip 0.1in
  \centering
  \begin{tabular}{lllll}
    \toprule
    %\multicolumn{2}{c}{Part}                   \\
    %\cmidrule(r){1-2}
    type of bias & total & pos. & neg. & p-value \\
    \midrule
    $\bias_{\text{EO}}^S$ & 0.161 & 0\% &  100\% & <0.01\\ 
    $\bias_{\text{PE}}^S$ & 0.154 & 0\% &  100\% & <0.01 \\ 
    $\bias_{\text{CALI}}^S$ & 0.034  & 79\% & 21\% & 0.30\\ 
    \midrule
    $\bias_{\text{ROC}}$ & 0.016 & 46\%  & 54\% & 0.31 \\ 
    $\bias_{\text{xROC}}$ & 0.273 & 0\% & 100\% & <0.01 \\ 
    \bottomrule
  \end{tabular}
  \vskip -0.1in
\end{table}

\subsection{German Credit Data}
Moreover, we trained two logistic regression scores on the German Credit Risk dataset ($n=1000$) to predict if a borrower belongs to the good risk class. The first model \emph{LogR} uses all available nine predictors including the feature \emph{sex}, which we choose as protected attribute. For the second score \emph{LogR (debiased)}, the protected attribute was removed from the model input. We set \emph{female} as the expected discriminated group. The scores achieve an AUROC of 0.772 and 0.771.

Compared to COMPAS, the separation biases of both models are lower (all below 0.1) whereas the calibration biases are higher (close to 0.3).
Removing the attribute decreases the separation bias (Table \ref{gcd-table}), while it slightly increases the calibration bias. Note that while \emph{LogR} contains bias to the detriment of female, the debiased model predominantly favors female over male. This demonstrates the use and importance of the split into positive and negative components introduced in \ref{sec:pos-neg-components}. 

\begin{table*}
  \caption{Gender bias of logistic regression (trained with and without sex) scores on German Credit Risk dataset; positive and negative component from the point of view of female persons.}
  \label{gcd-table}
  \vskip 0.1in
  \centering
  \begin{tabular}{lllllllll}
    \toprule
    &\multicolumn{4}{c}{LogR} & \multicolumn{4}{c}{LogR (debiased)}  \\
    \cmidrule(r){2-5}\cmidrule(r){6-9}
    type of bias & total bias & pos. & neg. & p-value & total bias & pos. & neg. & p-value\\
    \midrule
    $\bias_{\text{EO}}^S$ & 0.083 & 1\% &  99\% & 0.04 & 0.048 & 93\% &  7\% & 0.32\\ 
    $\bias_{\text{PE}}^S$ & 0.092 & 0\% &  100\% & 0.09 &  0.025 & 62\% &  38\% & 0.99\\ 
     $\bias_{\text{CALI}}^S$ & 0.291 & 46\% &  54\% & 0.35 & 0.299 & 58\% &  42\% & 0.26\\ 
    \bottomrule
  \end{tabular}
  
  % \begin{tabular}{llllll}
  %   \toprule
  %   %\multicolumn{2}{c}{Part}                   \\
  %   %\cmidrule(r){1-2}
  %   type of bias & Model& total bias & pos. & neg. & p-value \\
  %   \midrule
  %   $\bias_{\text{EO}}^S$ & LogR & 0.083 & 0\% &  100\% & 0.05\\ 
  %    & LogR (debiased) & 0.048 & 93\% &  7\% & 0.29\\ 
  %    \midrule
  %   $\bias_{\text{PE}}^S$ & LogR & 0.092 & 0\% &  100\% & 0.12\\ 
  %    & LogR (debiased) &  0.010 & 62\% &  38\% & 1\\ 
  %    \midrule
  %    $\bias_{\text{CALI}}^S$ & LogR & 0.291 & 46\% &  54\% & 0.28\\ 
  %    & LogR (debiased) & 0.299 & 58\% &  42\% & 0.29\\ 
  %   \bottomrule
  % \end{tabular}
  \vskip -0.1in
\end{table*}

\subsection{UCI Adult}

\begin{figure*}[ht]
    \vskip 0.2in
    \centering
    \begin{minipage}{0.48\textwidth}
        \centering
        \includegraphics[width=\textwidth]{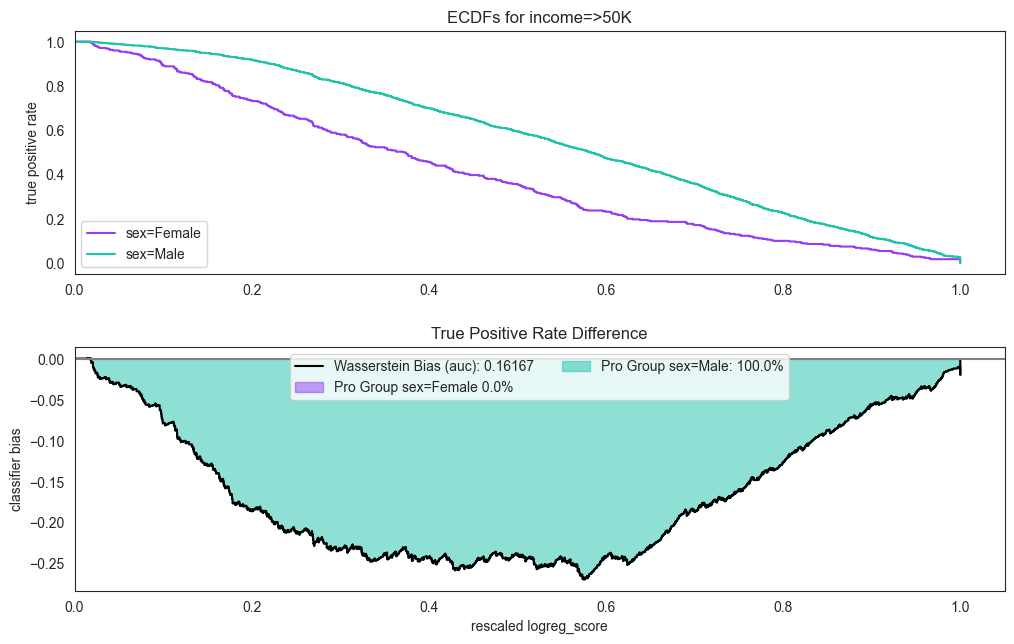}
        \caption*{(a) $\bias_{\text{EO}}^\mathcal{U}$}
        \label{adult-WEO-U}
    \end{minipage}\hfill
    \begin{minipage}{0.48\textwidth}
        \centering
        \includegraphics[width=\textwidth]{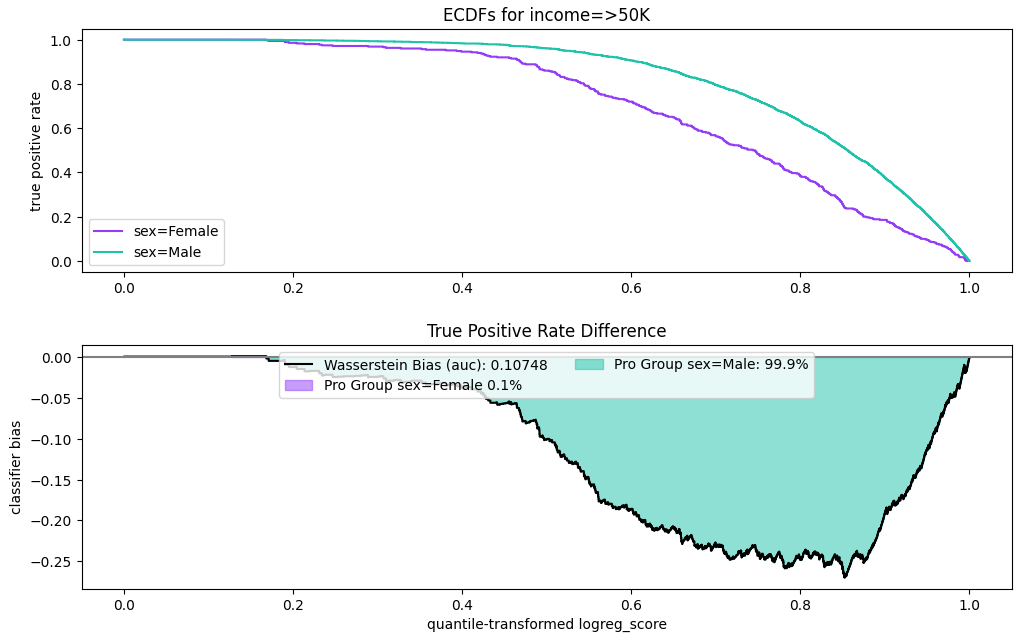}
        \caption*{(b) $\bias_{\text{EO}}^S$}
        \label{adult-WEO-S}
    \end{minipage}
    \caption{Equal opportunity biases $\bias_{\text{EO}}^\mathcal{U}$ and $\bias_{\text{EO}}^S$ of the logistic regression model trained on the Adult dataset. Each of the biases is equal to the area under the curve of the true positive rate difference. The area is colored according to the group for which the bias part is favorable.}
    \label{adult-WEO-U-and-S}
    \vskip -0.1in
\end{figure*}

Moreover, we used the UCI Adult dataset ($n=32561$) to train three different scores that predict the probability of the income being above 50k\$. Again, we choose \emph{sex} as the protected attribute and \emph{female} as the expected discriminated group. As before, a logistic regression was trained including (\emph{logR}) and excluding (\emph{logR (debiased)}) the protected attribute \emph{sex}. Moreover, an XGBoost model (\emph{XGB}), was trained with the complete feature set. XGB is known as one of the best performing methods on tabular data~\cite{DBLP:journals/corr/abs-2106-03253}. The logistic regression achieved an AUROC of 0.898 with and of 0.897 without the protected attribute, the XGB model achieved an AUROC of 0.922 on the testset. Resulting biases are shown in Table \ref{adult-table}, with the lowest bias in bold. 

Removing the protected attribute from the model input improves all biases of LogR except $\bias_{\text{ROC}}$ but separation biases are still against female while the calibration bias of the debiased model is predominantly in favor of female. XGB outperforms the logistic regression model that was trained on the same data in terms of fairness. In half of the cases, the bias of the XGB model is even smaller than the bias of logR (debiased). Here, due to the high sample size, all biases are statistically significant. We see a difference between $\bias^{\mathcal{U}}$ and $\bias^{S}$ that is due to the skewed score distributions on the imbalanced dataset (appendix Fig. %\ref{adult-scores-logreg}-\ref{adult-scores-xgb}
C1-C3): in general rate differences in the range of low scores are weighted higher for $\bias^{S}$ as they effect more people (Fig. \ref{adult-WEO-U-and-S}). Note that $\bias_{\text{ROC}}$ is in favor of female persons: Looking only at groupwise ROC curves ($\bias_\text{ROC}$) suggests an advantage for females. However, female persons experience lower true- and false positive rates at every possible threshold that is chosen independently of the group, as $\bias^S_{\text{EO}}$ and $\bias^S_{\text{PE}}$ clearly show.

% \begin{figure}
%   \centering
%   \includegraphics[width=0.7\columnwidth]{figs/adult_scores.png}
%   \caption{Score distribution of Adult Model conditioned on income class.}
% \end{figure}

\begin{table*}
  \caption{Gender bias of logistic regression (trained with and without sex) and XGBoost on Adult dataset; positive and negative component from the point of view of female persons. Each permutation tests gives $p<0.01$.}
  \label{adult-table}
  \vskip 0.1in
  \centering
  \begin{tabular}{llllllllll}
    \toprule
    &\multicolumn{3}{c}{LogR} &  \multicolumn{3}{c}{LogR (debiased)} &  \multicolumn{3}{c}{XGB}                \\
    \cmidrule(r){2-4}\cmidrule(r){5-7}\cmidrule(r){8-10}
    type of bias &  total bias & pos. & neg. & total bias & pos. & neg. & total bias & pos. & neg.\\
    \midrule
    $\bias_{\text{EO}}^S$ & 0.107 & 0\% &  100\% & 
        0.069 & 0\% &  100\% &
        \textbf{0.057} & 1\% &  99\% \\ 
    $\bias_{\text{PE}}^S$ & 0.164 & 0\% &  100\% & 
        \textbf{0.121} & 0\% &  100\% &
        0.143 & 0\% &  100\%\\ 
     $\bias_{\text{CALI}}^S$ & 0.052 & 22\% &  78\% &
        \textbf{0.045} & 55\% &  45\% &
        0.050 & 52\% &  48\%\\ 
    \midrule
    $\bias_{\text{ROC}}$ & 0.050  & 98\%  & 2\%  &
        0.051 & 98\% &  2\% &
        \textbf{0.033} & 98\% &  2\%\\
    $\bias_{\text{xROC}}$ & 0.205  & 0\% & 100\% &
        0.151 & 0\% & 100\% &
        \textbf{0.129} & 0\% & 100\% \\ 
    \midrule
    $\bias_{\text{EO}}^\mathcal{U}$ & 0.161  & 0\% & 100\% &
        0.104 & 0\% & 100\%  &
        \textbf{0.087} & 0\% & 100\% \\
    $\bias_{\text{PE}}^\mathcal{U}$& 0.118 &  0\% & 100\% &
        \textbf{0.098} & 0\% & 100\% & 
        0.101 & 0\% & 100\% \\ 
    $\bias_{\text{CALI}}^\mathcal{U}$&  0.105  &  20\%&  80\% &
        \textbf{0.102} & 50\% & 50\% &
         0.138 & 62\% & 38\%\\ 
    \bottomrule
  \end{tabular}
  \vskip -0.1in
\end{table*}

\subsection{Synthetic Data}

\begin{figure}[bt]
    \centering
    \includegraphics[width=\columnwidth]{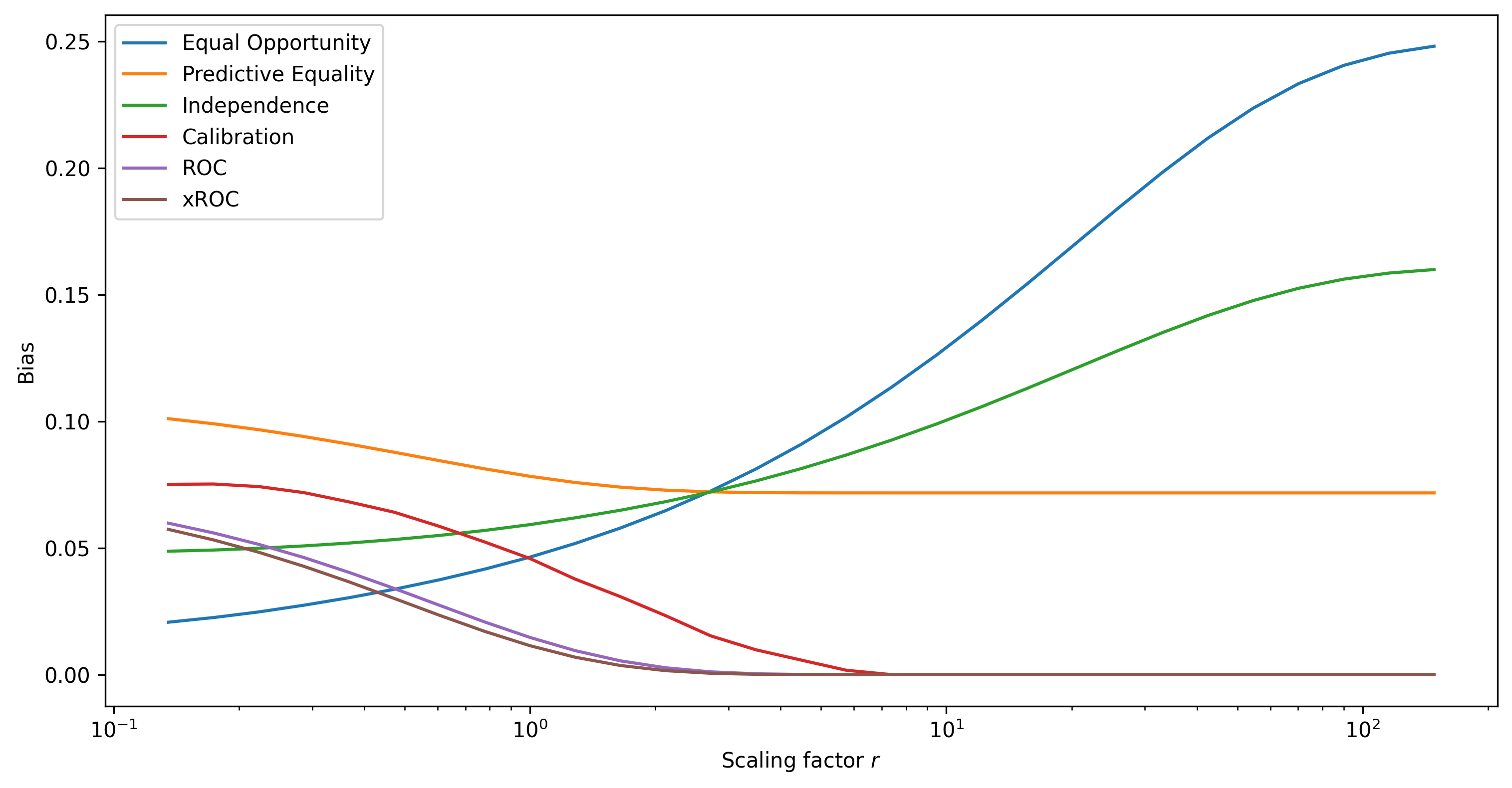}
    \caption{Changing bias measures with increasing distance between the groups and classes.}
    \label{fig:synthetic}
\end{figure}

In order to evaluate how different metrics change when the bias changes, we make use of synthetic datasets. This allows us to change the bias and observe the effect on the different metrics. For this reason, we sample $S_{a0}, S_{a1}, S_{b0}$ and $S_{b1}$ independently from four Gaussian distributions. 

Utilizing a scaling factor $r>0$, we set the following distributions: $S_{a0}\sim \mathcal{N}(1\cdot r, 0.6^2\cdot r)$, $S_{a1}\sim \mathcal{N}(-1, 0.5)$, $S_{b0}\sim \mathcal{N}(1.2\cdot r, 0.75^2\cdot r)$ and $S_{b1}\sim \mathcal{N}(-1.3, 0.6)$. Note that scores of the positive class of both groups move further away from each others with increasing $r$ (i.e. an increasing equal opportunity bias), while the negative class stays unchanged. The effect of this increasing difference can be seen in Fig.~\ref{fig:synthetic}.

We chose this setting to demonstrate the implications of Theorem \ref{thm:final_inequality_wasserstein_roc} and Corollary \ref{cor:zeroseparation}. Even though the difference between $S_{a0}$ and $S_{b0}$ grows, both ROC and xROC are unable to detect this disparity.

\section{Discussion and Outlook}
\label{section:outlook}
In this paper, we introduced a family of standardized group disparity measures for continuous risk scores that have an intuitive interpretation and theoretical grounding based on the Wasserstein distance. We derived their relation to well-established parity concepts and to ROC-based measures and we proved, that reducing the proposed separation biases is a stronger objective than reducing ROC-based measures and, hence, is better suited to cover different sorts of bias. Moreover, we demonstrated the practical application on fairness benchmark datasets. Our results show that removing information about the attribute influences the fairness of a model and also which group is affected by it. They also show that debiasing often leads to a shift between different bias types and should be monitored carefully. XGBoost results may indicate that flexible models can produce fairer results than simpler models.
The results of our experiments can serve as a starting point for a comprehensive comparison of score models (in terms of bias) and debiasing methods for such models. This work would then provide evaluation metrics for such a comparison.

The proposed measures generalize rate differences from classification tasks to entire score models. As a future extension, a generalization of rate ratios is another option that is to be explored. Moreover, the discussed decision model errors (TPR/FPR/Calibration) could be summed or related to each other (i.e., $\nicefrac{\text{TPR}}{\text{FPR}}$) to create further disparity measures. Note also, that the given definitions of the classifier biases are based on the $l_1$-norm. Especially when used for bias mitigation, that we did not cover here, it may also be useful to replace the $l_1$-norm by $l_p$ with $p>1$, especially $l_2$ or $l_\infty$, to penalize large disparities more than small ones. However, the score bias is then no longer a Wasserstein-distance. Another option is to use the Wasserstein-$p$-distance with $p>1$.  Typically, the outcome of fairness analyses is to assess whether certain groups are discriminated against by a score model. All the proposed disparity measures can be used to assess the group disparity of the errors made by the model. While parity, i.e. a small bias, can be taken as a sign that there is no algorithmic unfairness in a sample with respect to a particular type of error, not all disparities are discriminatory. For practical applications we propose not to use hard thresholds to decide whether a model is fair or unfair. If needed, such thresholds can be chosen similarly to the thresholds for classification biases and should be task-specific. Once a high bias is detected, the causes of the disparities should be analyzed in detail to decide for follow-up actions. The relation to the field of causal fairness criteria (i.e. \cite{nilforoshan2022causal, zhang2018fairness, makhlouf2020survey}) is out of scope of this manuscript. Further studies should investigate the relation and how they can be used to perform follow-up analyses in case of significant group disparities. 

\section*{Impact Statement}
This paper extends the existing ways of measuring bias in the context of continuous scores. 
% If our methodology will be used in production systems it will have societal consequences when such a bias is found. These consequences will hopefully lead to fairer machine learning systems. 
% We hope that our method will in the future be used to report existing bias; especially in situations where the score itself has to be considered (e.g. credit scores) and not only a binary decision build on top of it. As a consequence this work can potentially contribute to a discussion of bias in scoring systems and, hence, lead to the development and use of bias-reduced (and hence fairer) scores.
% Contribution by Gökce:
The aim is to report existing bias, particularly in situations where the score itself must be considered, such as credit scores, rather than just a binary decision based on it. This work has the potential to contribute to the discussion of bias in scoring systems and lead to the development and use of fairer, bias-reduced scores.

\bibliographystyle{icml2024}
\bibliography{References}

\begin{thebibliography}{46}
\providecommand{\natexlab}[1]{#1}
\providecommand{\url}[1]{\texttt{#1}}
\expandafter\ifx\csname urlstyle\endcsname\relax
  \providecommand{\doi}[1]{doi: #1}\else
  \providecommand{\doi}{doi: \begingroup \urlstyle{rm}\Url}\fi

\bibitem[Agarwal et~al.(2019)Agarwal, Dud{\'i}k, and Wu]{Agarwal2019FairRQ}
Agarwal, A., Dud{\'i}k, M., and Wu, Z.~S.
\newblock Fair regression: Quantitative definitions and reduction-based algorithms.
\newblock \emph{ArXiv}, abs/1905.12843, 2019.
\newblock URL \url{https://api.semanticscholar.org/CorpusID:170079300}.

\bibitem[Alvarado \& Waern(2018)Alvarado and Waern]{alvaradoAlgorithmicExperienceInitial2018}
Alvarado, O. and Waern, A.
\newblock Towards {{Algorithmic Experience}}: {{Initial Efforts}} for {{Social Media Contexts}}.
\newblock In \emph{Proceedings of the 2018 {{CHI Conference}} on {{Human Factors}} in {{Computing Systems}}}, pp.\  1--12, {Montreal QC Canada}, April 2018. {ACM}.
\newblock ISBN 978-1-4503-5620-6.
\newblock \doi{10.1145/3173574.3173860}.
\newblock URL \url{https://dl.acm.org/doi/10.1145/3173574.3173860}.

\bibitem[Barocas et~al.(2019)Barocas, Hardt, and Narayanan]{barocasFairnessMachineLearning2019}
Barocas, S., Hardt, M., and Narayanan, A.
\newblock \emph{Fairness and {{Machine Learning}}: {{Limitations}} and {{Opportunities}}}.
\newblock {fairmlbook.org}, 2019.
\newblock URL \url{http://www.fairmlbook.org}.

\bibitem[Beutel et~al.(2019)Beutel, Chen, Doshi, Qian, Wei, Wu, Heldt, Zhao, Hong, Chi, and Goodrow]{beutelFairnessRecommendationRanking2019}
Beutel, A., Chen, J., Doshi, T., Qian, H., Wei, L., Wu, Y., Heldt, L., Zhao, Z., Hong, L., Chi, E.~H., and Goodrow, C.
\newblock Fairness in {{Recommendation Ranking}} through {{Pairwise Comparisons}}.
\newblock In \emph{Proceedings of the 25th {{ACM SIGKDD International Conference}} on {{Knowledge Discovery}} \& {{Data Mining}}}, pp.\  2212--2220, {Anchorage AK USA}, July 2019. {ACM}.
\newblock ISBN 978-1-4503-6201-6.
\newblock \doi{10.1145/3292500.3330745}.
\newblock URL \url{https://dl.acm.org/doi/10.1145/3292500.3330745}.

\bibitem[Borkan et~al.(2019)Borkan, Dixon, Sorensen, Thain, and Vasserman]{borkanNuancedMetricsMeasuring2019}
Borkan, D., Dixon, L., Sorensen, J., Thain, N., and Vasserman, L.
\newblock Nuanced {{Metrics}} for {{Measuring Unintended Bias}} with {{Real Data}} for {{Text Classification}}.
\newblock In \emph{Companion {{Proceedings}} of {{The}} 2019 {{World Wide Web Conference}}}, pp.\  491--500, {San Francisco USA}, May 2019. {ACM}.
\newblock ISBN 978-1-4503-6675-5.
\newblock \doi{10.1145/3308560.3317593}.
\newblock URL \url{https://dl.acm.org/doi/10.1145/3308560.3317593}.

\bibitem[Bucher(2017)]{bucherAlgorithmicImaginaryExploring2017}
Bucher, T.
\newblock The algorithmic imaginary: Exploring the ordinary affects of {{Facebook}} algorithms.
\newblock \emph{Information, Communication \& Society}, 20\penalty0 (1):\penalty0 30--44, January 2017.
\newblock ISSN 1369-118X, 1468-4462.
\newblock \doi{10.1080/1369118X.2016.1154086}.
\newblock URL \url{https://www.tandfonline.com/doi/full/10.1080/1369118X.2016.1154086}.

\bibitem[Chouldechova(2017)]{chouldechovaFairPredictionDisparate2017b}
Chouldechova, A.
\newblock Fair {{Prediction}} with {{Disparate Impact}}: {{A Study}} of {{Bias}} in {{Recidivism Prediction Instruments}}.
\newblock \emph{Big Data}, 5\penalty0 (2):\penalty0 153--163, June 2017.
\newblock ISSN 2167-6461, 2167-647X.
\newblock \doi{10.1089/big.2016.0047}.
\newblock URL \url{http://www.liebertpub.com/doi/10.1089/big.2016.0047}.

\bibitem[Chzhen et~al.(2020)Chzhen, Denis, Hebiri, Oneto, and Pontil]{chzhenFairRegressionWasserstein2020}
Chzhen, E., Denis, C., Hebiri, M., Oneto, L., and Pontil, M.
\newblock Fair {{Regression}} with {{Wasserstein Barycenters}}.
\newblock In \emph{Proceedings of the 34th {{International Conference}} on {{Neural Information Processing Systems}}}, \{\vphantom\}{{NIPS}}'20\vphantom\{\}. {\{Curran Associates Inc.\}}, June 2020.
\newblock \doi{10.5555/3495724.3496338}.
\newblock URL \url{http://arxiv.org/abs/2006.07286}.

\bibitem[{Corbett-Davies} et~al.(2017){Corbett-Davies}, Pierson, Feller, Goel, and Huq]{corbett-daviesAlgorithmicDecisionMaking2017}
{Corbett-Davies}, S., Pierson, E., Feller, A., Goel, S., and Huq, A.
\newblock Algorithmic decision making and the cost of fairness.
\newblock In \emph{Proceedings of the 23rd Acm Sigkdd International Conference on Knowledge Discovery and Data Mining}, pp.\  797--806, June 2017.
\newblock \doi{10.1145/3097983.309809}.
\newblock URL \url{http://arxiv.org/abs/1701.08230}.

\bibitem[Datta et~al.(2014)Datta, Tschantz, and Datta]{datta2014automated}
Datta, A., Tschantz, M.~C., and Datta, A.
\newblock Automated experiments on ad privacy settings: A tale of opacity, choice, and discrimination.
\newblock \emph{arXiv preprint arXiv:1408.6491}, 2014.

\bibitem[DiCiccio et~al.(2020)DiCiccio, Vasudevan, Basu, Kenthapadi, and Agarwal]{diciccioEvaluatingFairnessUsing2020}
DiCiccio, C., Vasudevan, S., Basu, K., Kenthapadi, K., and Agarwal, D.
\newblock Evaluating {{Fairness Using Permutation Tests}}.
\newblock \emph{Proceedings of the 26th ACM SIGKDD International Conference on Knowledge Discovery \& Data Mining}, pp.\  1467--1477, August 2020.
\newblock URL \url{https://doi.org/10.1145/3394486.3403199}.

\bibitem[Dwork et~al.(2012)Dwork, Hardt, Pitassi, Reingold, and Zemel]{dworkFairnessAwareness2012}
Dwork, C., Hardt, M., Pitassi, T., Reingold, O., and Zemel, R.
\newblock Fairness {{Through Awareness}}.
\newblock \emph{ITCS '12: Proceedings of the 3rd Innovations in Theoretical Computer Science Conference}, pp.\  214--226, 2012.
\newblock \doi{10.1145/2090236.2090255}.
\newblock URL \url{http://arxiv.org/abs/1104.3913}.

\bibitem[Esteva et~al.(2017)Esteva, Kuprel, Novoa, Ko, Swetter, Blau, and Thrun]{estevaDermatologistlevelClassificationSkin2017}
Esteva, A., Kuprel, B., Novoa, R.~A., Ko, J., Swetter, S.~M., Blau, H.~M., and Thrun, S.
\newblock Dermatologist-level classification of skin cancer with deep neural networks.
\newblock \emph{Nature}, 542\penalty0 (7639):\penalty0 115--118, February 2017.
\newblock ISSN 0028-0836, 1476-4687.
\newblock \doi{10.1038/nature21056}.
\newblock URL \url{http://www.nature.com/articles/nature21056}.

\bibitem[Gardner et~al.(2019)Gardner, Brooks, and Baker]{gardnerEvaluatingFairnessPredictive2019}
Gardner, J., Brooks, C., and Baker, R.
\newblock Evaluating the {{Fairness}} of {{Predictive Student Models Through Slicing Analysis}}.
\newblock In \emph{Proceedings of the 9th {{International Conference}} on {{Learning Analytics}} \& {{Knowledge}}}, pp.\  225--234, {Tempe AZ USA}, March 2019. {ACM}.
\newblock ISBN 978-1-4503-6256-6.
\newblock \doi{10.1145/3303772.3303791}.
\newblock URL \url{https://dl.acm.org/doi/10.1145/3303772.3303791}.

\bibitem[Han et~al.(2023)Han, Jiang, Jin, Liu, Zou, Wang, and Hu]{hanRetiringDeltaDP2023}
Han, X., Jiang, Z., Jin, H., Liu, Z., Zou, N., Wang, Q., and Hu, X.
\newblock Retiring {$\Delta$DP}: New distribution-level metrics for demographic parity.
\newblock In \emph{{{arXiv}}:2301.13443}. {arXiv}, January 2023.
\newblock \doi{10.48550/arXiv.2301.13443}.
\newblock URL \url{http://arxiv.org/abs/2301.13443}.

\bibitem[Hardt et~al.(2016)Hardt, Price, and Srebro]{hardtEqualityOpportunitySupervised2016}
Hardt, M., Price, E., and Srebro, N.
\newblock Equality of {{Opportunity}} in {{Supervised Learning}}.
\newblock In \emph{Proceedings of the 30th {{International Conference}} on {{Neural Information Processing Systems}}}, 29, pp.\ ~9, 2016.
\newblock \doi{10.5555/3157382.3157469}.

\bibitem[Holstein et~al.(2018)Holstein, McLaren, and Aleven]{holsteinStudentLearningBenefits2018}
Holstein, K., McLaren, B.~M., and Aleven, V.
\newblock Student {{Learning Benefits}} of a {{Mixed-Reality Teacher Awareness Tool}} in {{AI-Enhanced Classrooms}}.
\newblock In Penstein~Ros{\'e}, C., {Mart{\'i}nez-Maldonado}, R., Hoppe, H.~U., Luckin, R., Mavrikis, M., {Porayska-Pomsta}, K., McLaren, B., and Du~Boulay, B. (eds.), \emph{Artificial {{Intelligence}} in {{Education}}}, volume 10947, pp.\  154--168. {Springer International Publishing}, {Cham}, 2018.
\newblock ISBN 978-3-319-93842-4 978-3-319-93843-1.
\newblock \doi{10.1007/978-3-319-93843-1_12}.
\newblock URL \url{http://link.springer.com/10.1007/978-3-319-93843-1_12}.

\bibitem[Hort et~al.(2023)Hort, Chen, Zhang, Harman, and Sarro]{hort2023bias}
Hort, M., Chen, Z., Zhang, J.~M., Harman, M., and Sarro, F.
\newblock Bias mitigation for machine learning classifiers: A comprehensive survey.
\newblock \emph{ACM Journal on Responsible Computing}, 2023.

\bibitem[Jiang et~al.(2020)Jiang, Pacchiano, Stepleton, Jiang, and Chiappa]{jiang20a}
Jiang, R., Pacchiano, A., Stepleton, T., Jiang, H., and Chiappa, S.
\newblock Wasserstein fair classification.
\newblock In Adams, R.~P. and Gogate, V. (eds.), \emph{Proceedings of The 35th Uncertainty in Artificial Intelligence Conference}, volume 115 of \emph{Proceedings of Machine Learning Research}, pp.\  862--872. PMLR, 22--25 Jul 2020.

\bibitem[Kallus \& Zhou(2019)Kallus and Zhou]{kallusFairnessRiskScores2019}
Kallus, N. and Zhou, A.
\newblock The {{Fairness}} of {{Risk Scores Beyond Classification}}: {{Bipartite Ranking}} and the {{XAUC Metric}}.
\newblock In Wallach, H., Larochelle, H., Beygelzimer, A., d'{Alch{\'e}-Buc}, F., Fox, E., and Garnett, R. (eds.), \emph{Proceedings of the 33rd {{International Conference}} on {{Neural Information Processing Systems}}}, volume~32. {Curran Associates, Inc.}, 2019.

\bibitem[Kamishima et~al.(2012)Kamishima, Akaho, Asoh, and Sakuma]{kamishimaFairnessAwareClassifierPrejudice2012}
Kamishima, T., Akaho, S., Asoh, H., and Sakuma, J.
\newblock Fairness-{{Aware Classifier}} with {{Prejudice Remover Regularizer}}.
\newblock In Hutchison, D., Kanade, T., Kittler, J., Kleinberg, J.~M., Mattern, F., Mitchell, J.~C., Naor, M., Nierstrasz, O., Pandu~Rangan, C., Steffen, B., Sudan, M., Terzopoulos, D., Tygar, D., Vardi, M.~Y., Weikum, G., Flach, P.~A., De~Bie, T., and Cristianini, N. (eds.), \emph{Machine {{Learning}} and {{Knowledge Discovery}} in {{Databases}}}, volume 7524, pp.\  35--50. {Springer Berlin Heidelberg}, {Berlin, Heidelberg}, 2012.
\newblock ISBN 978-3-642-33485-6 978-3-642-33486-3.
\newblock \doi{10.1007/978-3-642-33486-3_3}.
\newblock URL \url{http://link.springer.com/10.1007/978-3-642-33486-3_3}.

\bibitem[Kleinberg et~al.(2018)Kleinberg, Mullainathan, and Raghavan]{kleinbergInherentTradeOffsFair2018}
Kleinberg, J., Mullainathan, S., and Raghavan, M.
\newblock Inherent {{Trade-Offs}} in the {{Fair Determination}} of {{Risk Scores}}.
\newblock In \emph{{{arXiv}}:1609.05807}, 2018.
\newblock URL \url{http://arxiv.org/abs/1609.05807}.

\bibitem[K{\"o}chling \& Wehner(2020)K{\"o}chling and Wehner]{Kchling2020DiscriminatedBA}
K{\"o}chling, A. and Wehner, M.~C.
\newblock Discriminated by an algorithm: a systematic review of discrimination and fairness by algorithmic decision-making in the context of hr recruitment and hr development.
\newblock \emph{Business Research}, 2020.

\bibitem[Kozodoi et~al.(2022)Kozodoi, Jacob, and Lessmann]{kozodoi2022fairness}
Kozodoi, N., Jacob, J., and Lessmann, S.
\newblock Fairness in credit scoring: Assessment, implementation and profit implications.
\newblock \emph{European Journal of Operational Research}, 297\penalty0 (3):\penalty0 1083--1094, 2022.

\bibitem[Kumar et~al.(2019)Kumar, Liang, and Ma]{kumarVerifiedUncertaintyCalibration2019}
Kumar, A., Liang, P., and Ma, T.
\newblock Verified {{Uncertainty Calibration}}.
\newblock In \emph{Advances in {{Neural Information Processing SystemsNeurIPS}}}, volume~32. {arXiv}, 2019.
\newblock URL \url{http://arxiv.org/abs/1909.10155}.

\bibitem[{Kwegyir-Aggrey} et~al.(2021){Kwegyir-Aggrey}, Santorella, and Brown]{kwegyir-aggreyEverythingRelativeUnderstanding2021}
{Kwegyir-Aggrey}, K., Santorella, R., and Brown, S.~M.
\newblock Everything is {{Relative}}: {{Understanding Fairness}} with {{Optimal Transport}}.
\newblock \emph{arXiv:2102.10349 [cs]}, February 2021.
\newblock URL \url{http://arxiv.org/abs/2102.10349}.

\bibitem[Larson et~al.(2016)Larson, Mattu, Kirchner, and Angwin]{larson2016we}
Larson, J., Mattu, S., Kirchner, L., and Angwin, J.
\newblock How we analyzed the compas recidivism algorithm.
\newblock \emph{ProPublica (5 2016)}, 9\penalty0 (1):\penalty0 3--3, 2016.

\bibitem[Liu et~al.(2019)Liu, Simchowitz, and Hardt]{liuImplicitFairnessCriterion2019}
Liu, L.~T., Simchowitz, M., and Hardt, M.
\newblock The implicit fairness criterion of unconstrained learning.
\newblock In \emph{International {{Conference}} on {{Machine Learning}}}, pp.\  4051--4060. {arXiv}, January 2019.
\newblock URL \url{http://arxiv.org/abs/1808.10013}.

\bibitem[Makhlouf \& Zhioua(2021)Makhlouf and Zhioua]{makhloufApplicabilityMachineLearning2021}
Makhlouf, K. and Zhioua, S.
\newblock On the {{Applicability}} of {{Machine Learning Fairness Notions}}.
\newblock \emph{ACM SIGKDD Explorations Newsletter}, 1\penalty0 (23):\penalty0 14--23, 2021.
\newblock URL \url{https://doi.org/10.1145/3468507.3468511}.

\bibitem[Makhlouf et~al.(2020)Makhlouf, Zhioua, and Palamidessi]{makhlouf2020survey}
Makhlouf, K., Zhioua, S., and Palamidessi, C.
\newblock Survey on causal-based machine learning fairness notions.
\newblock \emph{arXiv preprint arXiv:2010.09553}, 2020.

\bibitem[Miroshnikov et~al.(2021)Miroshnikov, Kotsiopoulos, Franks, and Kannan]{miroshnikovModelagnosticBiasMitigation2021}
Miroshnikov, A., Kotsiopoulos, K., Franks, R., and Kannan, A.~R.
\newblock Model-agnostic bias mitigation methods with regressor distribution control for {{Wasserstein-based}} fairness metrics.
\newblock \emph{arXiv:2111.11259 [cs, math]}, November 2021.
\newblock URL \url{http://arxiv.org/abs/2111.11259}.

\bibitem[Miroshnikov et~al.(2022)Miroshnikov, Kotsiopoulos, Franks, and Ravi~Kannan]{miroshnikovWassersteinbasedFairnessInterpretability2022}
Miroshnikov, A., Kotsiopoulos, K., Franks, R., and Ravi~Kannan, A.
\newblock Wasserstein-based fairness interpretability framework for machine learning models.
\newblock \emph{Machine Learning}, 111\penalty0 (9):\penalty0 3307--3357, September 2022.
\newblock ISSN 0885-6125, 1573-0565.
\newblock \doi{10.1007/s10994-022-06213-9}.
\newblock URL \url{https://link.springer.com/10.1007/s10994-022-06213-9}.

\bibitem[Mitchell et~al.(2021)Mitchell, Potash, Barocas, D'Amour, and Lum]{mitchellAlgorithmicFairnessChoices2021}
Mitchell, S., Potash, E., Barocas, S., D'Amour, A., and Lum, K.
\newblock Algorithmic {{Fairness}}: {{Choices}}, {{Assumptions}}, and {{Definitions}}.
\newblock \emph{Annual Review of Statistics and Its Application}, 8\penalty0 (1):\penalty0 141--163, March 2021.
\newblock ISSN 2326-8298, 2326-831X.
\newblock \doi{10.1146/annurev-statistics-042720-125902}.
\newblock URL \url{https://www.annualreviews.org/doi/10.1146/annurev-statistics-042720-125902}.

\bibitem[Nilforoshan et~al.(2022)Nilforoshan, Gaebler, Shroff, and Goel]{nilforoshan2022causal}
Nilforoshan, H., Gaebler, J.~D., Shroff, R., and Goel, S.
\newblock Causal conceptions of fairness and their consequences.
\newblock In \emph{International Conference on Machine Learning}, pp.\  16848--16887. PMLR, 2022.

\bibitem[Pleiss et~al.(2017)Pleiss, Raghavan, Wu, Kleinberg, and Weinberger]{pleiss2017fairness}
Pleiss, G., Raghavan, M., Wu, F., Kleinberg, J., and Weinberger, K.~Q.
\newblock On fairness and calibration.
\newblock \emph{Advances in neural information processing systems}, 30, 2017.

\bibitem[Rader \& Gray(2015)Rader and Gray]{raderUnderstandingUserBeliefs2015}
Rader, E. and Gray, R.
\newblock Understanding {{User Beliefs About Algorithmic Curation}} in the {{Facebook News Feed}}.
\newblock In \emph{Proceedings of the 33rd {{Annual ACM Conference}} on {{Human Factors}} in {{Computing Systems}}}, pp.\  173--182, {Seoul Republic of Korea}, April 2015. {ACM}.
\newblock ISBN 978-1-4503-3145-6.
\newblock \doi{10.1145/2702123.2702174}.
\newblock URL \url{https://dl.acm.org/doi/10.1145/2702123.2702174}.

\bibitem[Saravanakumar(2021)]{saravanakumarImpossibilityTheoremMachine2021}
Saravanakumar, K.~K.
\newblock The {{Impossibility Theorem}} of {{Machine Fairness}} -- {{A Causal Perspective}}, January 2021.
\newblock URL \url{http://arxiv.org/abs/2007.06024}.

\bibitem[Schefzik et~al.(2021)Schefzik, Flesch, and Goncalves]{schefzikFastIdentificationDifferential2021}
Schefzik, R., Flesch, J., and Goncalves, A.
\newblock Fast identification of differential distributions in single-cell {{RNA-sequencing}} data with {{waddR}}.
\newblock \emph{Bioinformatics}, 37\penalty0 (19):\penalty0 3204--3211, October 2021.
\newblock ISSN 1367-4803, 1460-2059.
\newblock \doi{10.1093/bioinformatics/btab226}.
\newblock URL \url{https://academic.oup.com/bioinformatics/article/37/19/3204/6207964}.

\bibitem[Shwartz{-}Ziv \& Armon(2021)Shwartz{-}Ziv and Armon]{DBLP:journals/corr/abs-2106-03253}
Shwartz{-}Ziv, R. and Armon, A.
\newblock Tabular data: Deep learning is not all you need.
\newblock \emph{CoRR}, abs/2106.03253, 2021.
\newblock URL \url{https://arxiv.org/abs/2106.03253}.

\bibitem[Vogel et~al.(2021)Vogel, Bellet, and Cl{\'e}men{\c c}on]{vogelLearningFairScoring2021}
Vogel, R., Bellet, A., and Cl{\'e}men{\c c}on, S.
\newblock Learning {{Fair Scoring Functions}}: {{Bipartite Ranking}} under {{ROC-based Fairness Constraints}}.
\newblock In \emph{{{arXiv}}:2002.08159}, pp.\  784--792, February 2021.
\newblock URL \url{http://arxiv.org/abs/2002.08159}.

\bibitem[Wei et~al.(2023)Wei, Liu, Li, and Zha]{pmlr-v206-wei23a}
Wei, S., Liu, J., Li, B., and Zha, H.
\newblock Mean parity fair regression in rkhs.
\newblock In Ruiz, F., Dy, J., and van~de Meent, J.-W. (eds.), \emph{Proceedings of The 26th International Conference on Artificial Intelligence and Statistics}, volume 206 of \emph{Proceedings of Machine Learning Research}, pp.\  4602--4628. PMLR, 25--27 Apr 2023.
\newblock URL \url{https://proceedings.mlr.press/v206/wei23a.html}.

\bibitem[Yang et~al.(2022)Yang, Ko, Varshney, and Ying]{yang2022minimax}
Yang, Z., Ko, Y.~L., Varshney, K.~R., and Ying, Y.
\newblock Minimax auc fairness: Efficient algorithm with provable convergence.
\newblock \emph{arXiv preprint arXiv:2208.10451}, 2022.

\bibitem[Zafar et~al.(2017)Zafar, Valera, Gomez~Rodriguez, and Gummadi]{zafarFairnessDisparateTreatment2017}
Zafar, M.~B., Valera, I., Gomez~Rodriguez, M., and Gummadi, K.~P.
\newblock Fairness {{Beyond Disparate Treatment}} \& {{Disparate Impact}}: {{Learning Classification}} without {{Disparate Mistreatment}}.
\newblock In \emph{Proceedings of the 26th {{International Conference}} on {{World Wide Web}}}, pp.\  1171--1180, {Perth Australia}, April 2017. {International World Wide Web Conferences Steering Committee}.
\newblock ISBN 978-1-4503-4913-0.
\newblock \doi{10.1145/3038912.3052660}.
\newblock URL \url{https://dl.acm.org/doi/10.1145/3038912.3052660}.

\bibitem[Zhang \& Bareinboim(2018{\natexlab{a}})Zhang and Bareinboim]{zhang2018fairness}
Zhang, J. and Bareinboim, E.
\newblock Fairness in decision-making—the causal explanation formula.
\newblock In \emph{Proceedings of the AAAI Conference on Artificial Intelligence}, volume~32, 2018{\natexlab{a}}.

\bibitem[Zhang \& Bareinboim(2018{\natexlab{b}})Zhang and Bareinboim]{zhangEqualityOpportunityClassification2018}
Zhang, J. and Bareinboim, E.
\newblock Equality of {{Opportunity}} in {{Classification}}: {{A Causal Approach}}.
\newblock In Bengio, S., Wallach, H., Larochelle, H., Grauman, K., {Cesa-Bianchi}, N., and Garnett, R. (eds.), \emph{Advances in {{Neural Information Processing Systems}}}, volume~31. {Curran Associates, Inc.}, 2018{\natexlab{b}}.
\newblock URL \url{https://proceedings.neurips.cc/paper/2018/file/ff1418e8cc993fe8abcfe3ce2003e5c5-Paper.pdf}.

\bibitem[Zhao(2023)]{zhao2023costs}
Zhao, H.
\newblock Costs and benefits of fair regression.
\newblock \emph{Transactions on Machine Learning Research}, 2023.
\newblock ISSN 2835-8856.
\newblock URL \url{https://openreview.net/forum?id=v6anjyEDVW}.

\end{thebibliography}

%%%%%%%%%%%%%%%%%%%%%%%%%%%%%%%%%%%%%%%%%%%%%%%%%%%%%%%%%%%%%%%%%%%%%%%%%%%%%%%
%%%%%%%%%%%%%%%%%%%%%%%%%%%%%%%%%%%%%%%%%%%%%%%%%%%%%%%%%%%%%%%%%%%%%%%%%%%%%%%
% APPENDIX
%%%%%%%%%%%%%%%%%%%%%%%%%%%%%%%%%%%%%%%%%%%%%%%%%%%%%%%%%%%%%%%%%%%%%%%%%%%%%%%
%%%%%%%%%%%%%%%%%%%%%%%%%%%%%%%%%%%%%%%%%%%%%%%%%%%%%%%%%%%%%%%%%%%%%%%%%%%%%%%
\newpage
\appendix
\onecolumn
\section{Background definitions and results}
\subsection{Wasserstein-p-Distance}
\begin{definition}[Wasserstein-p-Distance]
The $p^{\text{th}}$ Wasserstein distance between two probability measures $\mu$  and $\nu$  in $\mathcal{P}_{p}(\mathbb{R}^d)$ is defined as

\begin{align}
W_{p}(\mu ,\nu ):=\left(\inf _{\gamma \in \Gamma (\mu ,\nu )}\int _{\mathbb{R}^d\times \mathbb{R}^d}d(x,y)^{p}\,\mathrm {d} \gamma (x,y)\right)^{1/p},
\end{align}
where $\Gamma (\mu ,\nu )$ denotes the collection of all measures on $\mathbb{R}^d\times \mathbb{R}^d$ with marginals $\mu$  and $\nu$ on the first and second factors respectively.
\end{definition}

\begin{corollary}
The Wasserstein metric may be equivalently defined by
\begin{align}
W_{p}(\mu ,\nu )=\left(\inf \operatorname {\mathbb{E}} {\big [}d(X,Y)^{p}{\big ]}\right)^{1/p},
\end{align}
where $\mathbb{E}[Z]$ denotes the expected value of a random variable $Z$ and the infimum is taken over all joint distributions of the random variables $X$ and $Y$ with marginals $\mu$ and $\nu$ respectively.
\end{corollary}

If $d=1$, the Wasserstein distance has a closed form. For this special case, we define $W$ as a measure between two random variables.
\begin{corollary}
Let $X$ and $Y$ be two random variables on $\mathbb{R}$ and let $F_X$ and $F_Y$ denote their cumulative distribution functions. Then
\begin{align}W_{p}(X ,Y) = \left(\int_0^1|F_X^{-1}(s)-F_Y^{-1}(s)|^p \,ds\right)^{\frac{1}{p}}
\end{align}
\end{corollary}
 
\begin{proposition} Properties of the Wasserstein-Distance for $d=1$:
\begin{enumerate}
    \item For any real number $a$, $W_p(aX, aY) = |a| W_p(X, Y ).$
    \item For any fixed vector $x$, $W_p(X + x, Y + x) = W_p(X, Y ).$
    \item For independent $X_1,\ldots, X_n$ and independent $Y_1,\ldots, Y_n$, $$W_p\big(\sum_{i=1}^n X_i, \sum_{i=1}^n Y_i\big) \leq \sum_{i=1}^n W_p(X_i, Y_i ).$$ %by Minkowski inequality
\end{enumerate}
\end{proposition}

\subsection{Special case: One-dimensional Wasserstein-1-Distance}
\begin{corollary}
\label{onedimW1}
If $p=1$ and $X,Y$ are random variables on $\mathbb{R}$ with cumulative distribution functions $F_X$ and $F_Y$, then
\begin{align}
    W_1(X,Y)&=\int_0^1 |F^{-1}_X(p)-F^{-1}_Y(p)|dp \\
    &= \int_\mathbb{R} |F_X(t)-F_Y(t)|dt. \label{W1_inverse}
\end{align}
\end{corollary} 

\begin{remark}
The Wasserstein-1-distance is not invariant under monotone transformations (for instance, under scale tranformations).
\end{remark}

 \begin{remark}
    The Wasserstein distance is insensitive to small wiggles. For example if $P$ is
     uniform on $[0, 1]$ and $Q$ has density $1+\sin(2\pi kx)$ on $[0, 1]$ then their Wasserstein distance is $\mathcal{O}(1/k)$.
 \end{remark}

\begin{theorem}[lower bound of $W_1$]
    \label{thm:mean_diff}
    The Wasserstein-distance is always greater or equal to the distance of the means:
    \begin{align}
    W_1(X,Y) \geq | \mathbb{E}[X]-\mathbb{E}[Y] | 
    \end{align}
\end{theorem}

\begin{proof}
By Jensen inequality, as norm is convex.
\end{proof}

\begin{theorem}[upper bound of $W_1$]
    For integers $p\leq q$,  
    \begin{align}
    W_p(X,Y) \leq W_q(X,Y), 
    \end{align}
    especially 
    \begin{align}
    W_1(X,Y) \leq W_q(X,Y) \quad \forall q \geq 1.
    \end{align}
\end{theorem}

\begin{proof}
By Jensen inequality, as $z\rightarrow z^{q/p}$ is convex.
% see also http://math.univ-lyon1.fr/~santambrogio/Wp.pdf, https://link.springer.com/chapter/10.1007/978-3-030-38438-8_2
\end{proof}

 \subsection{Wasserstein-Distance of Quantile-Transformed Variables}
%From Klaus Draft
\begin{definition}[Quantile-Transformed Wasserstein Distance]
    Let $X, Y, Z$ be random variables on $\mathbb{R}$ and let  $F_X, F_Y, F_Z: \mathbb{R}\to[0,1]$ denote their distribution functions and $f_Z$ denote the density of $Z$.
    The (by Z) quantile-transformed Wasserstein Distance is then given by:
    \begin{align}
        W_Z(X,Y) &:= W_1(F_Z(X),F_Z(Y)) \\
        & = \int_0^1\left|F_{F_Z(X)}(t)-F_{F_Z(Y)}(t)\right| \,dt \\
        &=\int_0^1\left|F_X(F_Z^{-1}(t))-F_Y(F_Z^{-1}(t))\right|\, dt\\
        &=\int_\mathbb{R}\left|F_X(s)-F_Y(s)\right|f_Z(s) \, ds 
        \label{Q-Wasserstein} 
    \end{align}
\end{definition}

\begin{proposition} Properties of the quantile-transformed Wasserstein-distance
\begin{enumerate}
    \item For any real number $a \neq 0$, $W_Z(aX, aY) = W_{Z/|a|}(X, Y).$
    \item For any fixed vector $x$, $W_Z(X + x, Y + x) = W_{Z-x}(X, Y ).$
    %\item For independent $X_1,\ldots, X_n$ and independent $Y_1,\ldots, Y_n$, $$W_p\big(\sum_{i=1}^n X_i, \sum_{i=1}^n Y_i\big) \leq \sum_{i=1}^n W_p(X_i, Y_i ).$$ %by Minkowski inequality
\end{enumerate}
\end{proposition}

\begin{remark}
The quantile-transformed Wasserstein-1-distance is invariant under monotone transformations, for instance, under scale tranformations: For $a>0$: 
\begin{align}
W_{Z}(X,Y) = W_{aZ}(aX, aY).
\end{align}
\end{remark}

% \begin{remark}[empirical calculation of the Wasserstein-1-distance]
%     The empricial Wasserstein-1-distance measures the weight $M_{ij}$ transferred from the $i$th position of the first distribution to the $j$th position of the second distribution and multiplies it by the distance $d_{ij}$:
    
%     $$ W_1(\hat{F}_1, \hat{F}_2) = \sum_{i=1}^n \sum_{j=1}^m  d_{ij} M_{ij} $$
    
%     For two ecdfs, the weight is $|F_1-F_2|$ and the distance is the interval length.
%     Sort two samples from $l=1$ to $i+j$
%     \begin{align}
%         W_1(\hat{F}_1, \hat{F}_2) &= \sum_{l=1}^{i+j-1} |x_{(l+1)}-x_{(l)}|\cdot |\hat{F}_1(x_{(l)})-\hat{F}_2(x_{(l)})|  \\
%          & = \sum_{l=1}^{i+j-1} \big(|x_{(l+1)}-x_{(l)}|\\
%          &\cdot |\frac{|\{{x \leq x_{(l+1)} \text{ with } A=a}\}|}{n}-\frac{|\{x \leq x_{(l+1)} \text{ with } A=b\}|}{m}|| \big)
%     \end{align}
%     $$ $$
%     Quantile transformation changes (only) the ground distance:
%     $$ W_1(\hat{F}_{A=a}, \hat{F}_{A=b}) = \sum_{l=1}^{i+j-1} |F(x_{(l+1)})-F(x_{(l)})| |\hat{F}_1(x_{(l)})-\hat{F}_2(x_{(l)})| $$
% \end{remark}
\subsection{Pushforward}

The pushforward of a measure along a measurable function assigns to a subset the original measure of the preimage under the function of that subset.
\begin{definition}
    Let $(X_1, \Sigma_1)$ and $(X_2, \Sigma_2)$ be two measurable spaces, $f: X_1 \rightarrow X_2$ a measurable function and $\mu: \Sigma_1 \rightarrow [0,\infty]$ a measure on $(X_1, \Sigma_1)$. The pushforward of $\mu$ is defined as 
    \begin{align}
    f\#\mu: \Sigma_2 \rightarrow [0,\infty], f\#\mu(A)=\mu(f^{-1}(A)) \, \forall A \in \Sigma_2
    \end{align}
\end{definition}

\begin{corollary}
\label{pushforward}
    Let again $(X_1, \Sigma_1)$ and $(X_2, \Sigma_2)$ be two measurable spaces, $f: X_1 \rightarrow X_2$ a measurable function and $\mu: \Sigma_1 \rightarrow [0,\infty]$ a measure on $(X_1, \Sigma_1)$.
    If $g$ is another measurable function on $X_2$, then
    \begin{equation}
        \label{changeofvariable}
        \int_{X_2} g \circ f \,d\mu = \int_{X_1} g\, d(f\#\mu)
    \end{equation}
\end{corollary}

\section{Complete proofs}
\begin{lemma}
\label{F_F_X}
If we quantile-transform a continuous random variable $X \in \mathbb{R}$ by its own distribution $F_X$, the result will follow a uniform distribution in $[0,1]$:
\begin{align}
F_X(X) \sim \mathcal{U}[0,1], \text{ so } F_{F_X}(x)=x.
\end{align}
\end{lemma}

 \begin{lemma}
\label{Lemma_F_X(Y)}
Let $X,Y$ be two random variables in $\mathbb{R}$ with cumulative distribution functions $F_X, F_Y$. The cumulative distribution function of a random variable $Z=F_X(Y)$ is given by $F_Y(F_X^{-1}(z))$:
    \begin{align}
        F_{F_X(Y)}(z) =F_Z(z) &= \mathbb{P}(Z\leq z)=\mathbb{P}(F_X(Y) \leq z)\\
        &=\mathbb{P}(Y \leq F_X^{-1}(z))=F_Y(F_X^{-1}(z)) \nonumber
    \end{align}
If $F_X$ and $F_Y$ are bijective and have the same support, then 
\begin{align}
F_{F_X(Y)} = F_{F_Y(X)}^{-1}.
\end{align}
\end{lemma}

\begin{proof}[Proof of Theorem \ref{uniformbias}]
For $x=\text{EO}$:

\begin{align}
     \bias_x^{\mathcal{U}}(S|A=a, S|A=b)&=\frac{1}{|\mathcal{S}|}\int_{\mathcal{S}}|\cbias_x(S|A=a, S|A=b; s)| ds\\
     & = \frac{1}{|\mathcal{S}|}\int_{\mathcal{S}}|\mathbb{P}(S> s|A=b, Y=0) - \mathbb{P}(S> s|A=a, Y=0)| ds \\
     & = \frac{1}{|\mathcal{S}|}\int_{\mathcal{S}}|(1-F_{b0}(s))-(1-F_{a0}(s))| ds \\
     & = \frac{1}{|\mathcal{S}|}\int_{\mathcal{S}}|F_{a0}(s)-F_{b0}(s)| ds \\
     & \overset{\eqref{onedimW1}}{=} \frac{1}{|\mathcal{S}|}\cdot W_1(S|A=a,Y=0, S|A=b,Y=0).
    %&\refeq{\ref{onedimW1}} \frac{1}{|\mathcal{S}|}\cdot W_1(S|A=a,Y=0, S|A=b,Y=0).
\end{align}

For $x=\text{PE}$ and $x=\text{\textit{IND}}$ the result follows similary.
(ii) follows from Theorem \ref{thm:mean_diff}.
 \end{proof}

\begin{proof}[Proof of Theorem \ref{quantilebias}]
For $x=\text{EO}$:
\begin{align}
        \bias_x^{S}(S|A=a, S|A=b)&=\int_\mathcal{S} |\cbias_x(S|A=a, S|A=b; s)| f(s) ds\\
        & \overset{\eqref{changeofvariable}}{=} \int_\mathcal{S} |\cbias_x(S|A=a, S|A=b; s)| d(F^{-1}\#\mu) \\
        %&\refeq{(\ref{changeofvariable})} \int_\mathcal{S} |\cbias_x(S|A=a, S|A=b; s)| d(F^{-1}\#\mu) \\
        &= \int_0^1 |\cbias_x(S|A=a, S|A=b; F^{-1}(p))| dp \\
        & \overset{\eqref{uniformbias}}{=} W_1(F_{F_S(S_{ay})}, F_{F_S(S_{by})}) \\
        %&\refeq{\ref{uniformbias}} W_1(F_{F_S(S_{ay})}, F_{F_S(S_{by})}) \\
        & \overset{\eqref{Lemma_F_X(Y)}}{=} W_1(F_{ay}\circ F_S^{-1}, F_{by}\circ F_S^{-1})
        %&\refeq{\ref{Lemma_F_X(Y)}} W_1(F_{ay}\circ F_S^{-1}, F_{by}\circ F_S^{-1})
\end{align}
For $x=\text{PE}$ and $x=\text{\textit{IND}}$ the result follows similary.
(ii) follows from Theorem \ref{thm:mean_diff}.
    
\end{proof}

\begin{proof}[Proof of Theorem \ref{thm:basicrocwasserstein}]
\begin{align}
        W_{Z}(X,Y) &= \int_0^1 \left|F_{F_Z(X)}(s) - F_{F_Z(Y)}(s)\right|ds\\
        & \overset{\eqref{Lemma_F_X(Y)}}{=} \int_0^1\left|F_X(F_Z^{-1}(s)) - F_Y(F_Z^{-1}(s))\right|ds\\ %Lemma_F_X(Y)
        %& \refeq{\ref{Lemma_F_X(Y)}} \int_0^1\left|F_X(F_Z^{-1}(s)) - F_Y(F_Z^{-1}(s))\right|ds\\ %Lemma_F_X(Y)
        &=\int_0^1\left|(1-F_X(F_Z^{-1}(1-r)) - (1-F_Y(F_Z^{-1}(1-r)))\right|dr\\
        &=\int_0^1\left|\ROC_{X,Z}(r)-\ROC_{Y,Z}(r)\right|dr
\end{align}
\end{proof}

\begin{proof}[Proof of Theorem \ref{prop:mixture}]
Results directly from Def. \ref{definition-bias-S} by using the additivity of the density in (\ref{density-wasserstein}). 
\end{proof}

\begin{proof}[Proof of Lemma \ref{lemma:wassersteinroc}]
Using Theorem \ref{thm:basicrocwasserstein} and $\ROC_{X,X}(r)=r$.
\end{proof}

Under additional assumptions, we can follow that a quantile-transformation by group $a$ and $b$ result in equal distances:
\begin{lemma}
    \label{lemma:exchange_quantile_function}
    If $S_{ay}$ and $S_{by}$ have bijective cdfs and identical support, then 
    \begin{align}
        W_{S_{ay}}(S_{ay},S_{by})=W_{S_{by}}(S_{ay},S_{by})=W_{S_y}(S_{ay},S_{by}).
    \end{align}
\end{lemma}

\begin{proof}[Proof of Lemma \ref{lemma:exchange_quantile_function}]
We show more general:
    \label{mixture_result}
    If $X,Y$ are two random variables on an interval $I$ in $\mathbb{R}$ with cdfs $F_X$ and $F_Y$ that are bijective on $I$
    $$W_X(X,Y) = W_Y(X,Y)$$ 
    %If they have not the same support (but overlapping), then a point mass builds up at the borders of the smaller support

By Lemma \ref{F_F_X} and by Lemma \ref{Lemma_F_X(Y)}, it follows
\begin{align}
    W_{X}(X,Y) &\stackrel{(\ref{Q-Wasserstein})}{=} \int_0^1\left|F_{F_X(X)}(t)-F_{F_X(Y)}(t)\right| dt \\
    &\stackrel{\ref{F_F_X}}{=}\int_0^1\left|t-F_{F_X(Y)}(t)\right| dt \\
    &\stackrel{\ref{Lemma_F_X(Y)}}{=}\int_0^1\left|t-F^{-1}_{F_Y(X)}(t)\right| dt \\
    &\stackrel{\ref{F_F_X}}{=}\int_0^1\left|F^{-1}_{F_Y(Y)}(t)-F^{-1}_{F_Y(X)}(t)\right| dt \\
    &\stackrel{(\ref{W1_inverse})}{=}W_{Y}(X,Y)
\end{align}

It follows for $Z=w_1X+w_2Y$:
    \begin{align}
        W_{Z}(X,Y) &= w_1 W_X(X,Y) + w_2 W_Y(X,Y)\\
        &= W_X(X,Y) = W_Y(X,Y) \nonumber
    \end{align}
and Lemma \ref{lemma:exchange_quantile_function} as a special case.
\end{proof}

Lemma \ref{lemma:exchange_quantile_function} implies that quantile-transformation can under the above assumptions be performed on either of the two groups or the whole sample with the same result. Under the same assumptions, $\ROC$, $\AUROC$ and $\gini$ become symmetrical, i.e. $\ROC_{S_{ay},S_{by}}=\ROC_{S_{by},S_{ay}}$.
% \begin{proof}[Proof of Corollary \ref{cor:zeroseparation}]
% \todo[author=Ann-Kristin]{not needed, follows from inequality. Kept because proof is simpler.}
% If both separation biases are zero it follows that $F_{A=a, Y=y}(t) = F_{A=b, Y=y}(t)$ almost everywhere and thus $f_{A=a, Y=y}(t) = f_{A=b, Y=y}(t)$ for almost all $t$ and for $y \in \{0,1\}$.
% \begin{align}
%         & \quad \bias_{\text{ROC}}(S|A=a, S|A=b) + \bias_{\text{xROC}}(S|A=a, S|A=b)\\ 
%         &= \int_0^1 |F_{b, 0}(s)f_{b, Y=1}(s)-F_{a, 0}(s)f_{a, Y=1}(s)| ds \\
%         &+ \int_0^1 |F_{b, 0}(s)f_{a, Y=1}(s)-F_{a, 0}(s)f_{b,Y=1}(s)| ds\\
%         &=0.
% \end{align}
% As a counterexample for the inverse, consider perfect separation between favorable and unfavorable outcomes but samples with favorable outcome from one group are ranked higher than from the other group. ROC and cross-ROC biases are then zero, as every combination of favorable and unfavorable samples leads to perfect separation, so all values are identical. However, equal opportunity bias would be maximal.
% \end{proof}

\begin{proof}[Proof of Theorem \ref{thm:sep_as_mixture}]
Using Proposition \ref{prop:mixture} and Lemma \ref{lemma:wassersteinroc}:
    \begin{align}
        & \bias_{\text{EO}}^{S}(S|A=a, S|A=b) = W_{S}(S_{a0},S_{b0}) \\
        & \overset{\eqref{prop:mixture}}{=} w_{a0} W_{S_{a0}}(S_{a0},S_{b0}) + w_{b0} W_{S_{b0}}(S_{a0},S_{b0})\\
        % & \refeq{{\ref{prop:mixture}}} w_{a0} W_{S_{a0}}(S_{a0},S_{b0}) + w_{b0} W_{S_{b0}}(S_{a0},S_{b0})\\
        &+ w_{a1} W_{S_{a1}}(S_{a0},S_{b0}) + w_{b1} W_{S_{b1}}(S_{a0},S_{b0}) \nonumber \\
        & \overset{\eqref{lemma:wassersteinroc}}{=} w_{a0}  \int \left|\ROC_{S_{a0},S_{b0}}(r)-r\right|dr + w_{b0} \int \left|\ROC_{S_{b0},S_{a0}}(r)-r\right|dr\\
        % & \refeq{{\ref{lemma:wassersteinroc}}}  w_{a0}  \int \left|\ROC_{S_{a0},S_{b0}}(r)-r\right|dr + w_{b0} \int \left|\ROC_{S_{b0},S_{a0}}(r)-r\right|dr\\
        & + w_{a1}\int |\ROC_{S_{a0}, S_{a1}}(r)-\ROC_{S_{b0}, S_{a1}}(r)| dr \nonumber\\
        & + w_{b1}\int |\ROC_{S_{a0}, S_{b1}}(r)-\ROC_{S_{b0}, S_{b1}}(r)| dr \nonumber
    \end{align} %\ref{lemma:wassersteinroc}

    For predictive equality, the result follows similarly.

    % \begin{align}
    %     &\bias_{\text{PE}}^{S}(S|A=a, S|A=b) = W_{S}(S_{a1},S_{b1}) \\
    %     & = w_{a0} W_{S_{a0}}(S_{a1},S_{b1}) + w_{b0} W_{S_{b0}}(S_{a1},S_{b1})  + w_{a1} W_{S_{a1}}(S_{a1},S_{b1}) + w_{b1} W_{S_{b1}}(S_{a1},S_{b1})\\
    %     &\refeq{{\ref{lemma:wassersteinroc}}} w_{a1} \int \left|\ROC_{S_{a1},S_{b1}}(r)-r\right|dr + w_{b1} \int \left|\ROC_{S_{b1},S_{a1}}(r)-r\right|dr \\
    %     & + w_{a0} \int |\ROC_{S_{a1}, S_{a0}}(r)-\ROC_{S_{b1}, S_{a0}}(r)| dr + w_{b0} \int |\ROC_{S_{a1}, S_{b0}}(r)-\ROC_{S_{b1}, S_{b0}}(r)| dr \\
    %     & \refeq{{\ref{mixture_result}}} w_{1} \int \left|\ROC_{S_{a1},S_{b1}}(r)-r\right|dr \\
    %     & + w_{a0} \int |\ROC_{S_{a1}, S_{a0}}(r)-\ROC_{S_{b1}, S_{a0}}(r)| dr + w_{b0} \int |\ROC_{S_{a1}, S_{b0}}(r)-\ROC_{S_{b1}, S_{b0}}(r)| dr \\
    % \end{align}
    \end{proof}

% \begin{proof}[Proof of Theorem \ref{thm:sep_as_mixture}]
%     \begin{align}
%         & \bias_{\text{EO}}(S|A=a, S|A=b)+ \bias_{\text{PE}}(S|A=a, S|A=b) =  W_{S}(S_{a0},S_{b0}) + W_{S}(S_{a1},S_{b1}) \\
%          \stackrel{(\ref{prop:mixture})}{=} &w_0 \cdot W_{S0}(S_{a0},S_{b0}) + w_1 \cdot W_{S1}(S_{a0},S_{b0}) + w_0 \cdot W_{S0}(S_{a1},S_{b1}) + w_1 \cdot W_{S1}(S_{a1},S_{b1}) \\
%          \stackrel{(\ref{lemma:wassersteinroc})}{=} &w_0  \int \left|\ROC_{S_{a0},S_{b0}}(r)-r\right|dr + w_1 \int \left|\ROC_{S_{a1},S_{b1}}(r)-r\right|dr \\
%          +\, &w_1\int |\ROC_{S_{a0}, S_1}(r)-\ROC_{S_{b0}, S_1}(r)| dr + w_0 \int |\ROC_{S_{a1}, S_0}(r)-\ROC_{S_{b1}, S_0}(r)| dr
%     \end{align}
%     \todo[author=Ann-Kristin]{use mixture of all components $a0, a1, b0, b1$ to get to the right xROC biases}
% \end{proof}
    \begin{proof}[Proof of Theorem \ref{thm:final_inequality_wasserstein_roc}]
    From theorem \ref{thm:sep_as_mixture} follows by triangle-Inequality:
    \begin{align}
        & \bias_{\text{EO}}^{S}(S|A=a, S|A=b)+ \bias_{\text{PE}}^{S}(S|A=a, S|A=b) =  W_{S}(S_{a0},S_{b0}) + W_{S}(S_{a1},S_{b1}) \\
        & \geq w_{0}  \int \left|\ROC_{S_{a0},S_{b0}}(r)-r\right|dr + w_{1} \int \left|\ROC_{S_{a1},S_{b1}}(r)-r\right| dr \\
         & + \min(w_{a0},w_{a1}) \cdot \bias_{\text{xROC}} + \min(w_{b0},w_{b1}) \cdot \bias_{\text{xROC}}  \nonumber\\
         & \geq w_{0}  \int \left|\ROC_{S_{a0},S_{b0}}(r)-r\right|dr + w_{1} \int \left|\ROC_{S_{a1},S_{b1}}(r)-r\right|dr \\ 
         & + 2 \min(w_{a0},w_{a1}, w_{b0}, w_{b1}) \cdot \bias_{\text{xROC}} \nonumber
    \end{align}
         and also
    \begin{align}
        & \bias_{\text{EO}}^{S}(S|A=a, S|A=b)+ \bias_{\text{PE}}^{S}(S|A=a, S|A=b) =  W_{S}(S_{a0},S_{b0}) + W_{S}(S_{a1},S_{b1}) \\
         & \geq w_{0}  \int \left|\ROC_{S_{a0},S_{b0}}(r)-r\right|dr + w_{1} \int \left|\ROC_{S_{a1},S_{b1}}(r)-r\right|dr \\
         & + \min(w_{a1},w_{b1}) \cdot \bias_{\text{ROC}} + \min(w_{a0},w_{b0}) \cdot \bias_{\text{ROC}}  \nonumber\\
         & \geq w_{0}  \int \left|\ROC_{S_{a0},S_{b0}}(r)-r\right|dr + w_{1} \int \left|\ROC_{S_{a1},S_{b1}}(r)-r\right|dr \\
         & + 2 \min(w_{a0},w_{a1}, w_{b0}, w_{b1}) \cdot \bias_{\text{ROC}} \nonumber
    \end{align}

    By additionally using Corollary \ref{lemma:rocwasserstein2}, we get 
    \begin{align}
     & \quad \bias_{\text{EO}}^{S}+ \bias_{\text{PE}}^{S}  =  W_{S}(S_{a0},S_{b0}) + W_{S}(S_{a1},S_{b1}) \\ &\geq \min(w_{a0},w_{a1}, w_{b0}, w_{b1}) \cdot (\bias_{\text{ROC}} + \bias_{\text{xROC}}) + \frac{w_0}{2} \gini(S_{a0}, S_{b0}) + \frac{w_1}{2}\gini(S_{a1}, S_{b1}) \nonumber
    \end{align}

    As $\frac{w_i}{2} \leq \frac{\min(w_{a0},w_{a1}, w_{b0}, w_{b1})}{2}$ for $i=0,1$, we can combine all weights to get

    \begin{align}
     & \quad \bias_{\text{EO}}^{S}+ \bias_{\text{PE}}^{S}  =  W_{S}(S_{a0},S_{b0}) + W_{S}(S_{a1},S_{b1}) \\ 
    &\geq \min(w_{a0},w_{a1}, w_{b0}, w_{b1}) \cdot (\bias_{\text{ROC}} + \bias_{\text{xROC}}) + \frac{w_0}{2} \gini(S_{a0}, S_{b0}) + \frac{w_1}{2}\gini(S_{a1}, S_{b1}) \\
     & \geq \frac{\min(w_{a0},w_{a1}, w_{b0}, w_{b1})}{2}(\bias_{\text{ROC}} + \bias_{\text{xROC}} + \gini(S_{a0}, S_{b0}) + \gini(S_{a1}, S_{b1}))
    \end{align}

    Note, that if $F_{ay}$ and $F_{by}$ have identical supports and permit an inverse, then $\gini(S_{ay},S_{by}) = \gini(S_{by},S_{ay})$. If this symmetry is not fulfilled, the minimum of both must be used on the right side.
\end{proof}

\begin{proof}[Proof of Theorem \ref{thm:zeroseparation2}]
Let $\bias_{\text{EO}}(S|A=a, S|A=b) =0$, it follows $F_{b0}=F_{a0}$ almost everywhere. Then
\begin{align}
        & \bias_{\text{ROC}}(S|A=a, S|A=b)\\
        &= \int_0^1 |F_{b0}(F_{b1}^{-1}(s))-F_{a0}(F_{a1}^{-1}(s))| ds \\
        &= \int_0^1 |F_{a0}(F_{b1}^{-1}(s))-F_{b0}(F_{a1}^{-1}(s))| ds\\
        &= \bias_{\text{xROC}}(S|A=a, S|A=b)
\end{align}
For predictive equality, the statement follows similarly.
\end{proof}

%It follows that $bias_{\text{PE}} = 0$ only if $\ROC = \xROC$.

% \begin{lemma} 
% \label{lemma:wassersteinrocineq2}
% It follows
% \begin{align}
%     \bias_{\text{ROC}}(A=a,A=b) & =  \int_0^1 |\ROC_{A=a}(p)-\ROC_{A=b}(p)| dp \\
%     & \leq  \int_0^1 |\ROC_{A=a}(p)-\xROC_{A=b, A=a}(p)| dp \\
%     &+ \int_0^1 |\xROC_{A=b, A=a}(p)-\ROC_{A=b}(p)| dp\\
%     & \leq \bias_{\text{EO}}^{S_{a,1}} +\bias_{\text{PE}}^{S_{b,0}} \\
%     & \leq \bias_{\text{EO}}^{S} +\bias_{\text{PE}}^{S}
% \end{align}
% and 
% \begin{align}
%     \bias_{\text{xROC}}(A=a,A=b) & =  \int_0^1 |\xROC_{S|A=a, S|A=b}(p)-\xROC_{A=b, A=a}(p)| dp \\
%     & \leq  \int_0^1 |\ROC_{A=a}(p)-\xROC_{A=b, A=a}(p)| dp \\
%     &+ \int_0^1 |\xROC_{S|A=a, S|A=b}(p)-\ROC_{A=a}(p)| dp\\
%     & \leq \bias_{\text{EO}}^{S_{a,1}} +\bias_{\text{PE}}^{S_{a,0}} \\
%     & \leq \bias_{\text{EO}}^{S} +\bias_{\text{PE}}^{S}.
% \end{align}
% \end{lemma}
\section{Experiments}
\setcounter{figure}{0}
\renewcommand{\thefigure}{C\arabic{figure}}

\setcounter{table}{0}
\renewcommand{\thetable}{C\arabic{table}}
% some results and details need to go here.
We perform experiments in python using the COMPAS dataset, the Adult dataset and the German Credit dataset. Empirical implementations of Wasserstein-distance (\emph{scipy.wasserstein\_distance}), calibration curves (\emph{sklearn.calibration.calibration\_curve}) and ROC curves (\emph{sklearn.metrics.roc\_curve}) were used.
\subsection{Statistical Testing}
We perform permutation tests \cite{diciccioEvaluatingFairnessUsing2020, schefzikFastIdentificationDifferential2021} with 1000 permutations and one pseudocount to determine the statistical significance of the calculated biases under the null hypothesis of group parity. The calibration biases were calculated using 50 bins.

\subsection{Details on COMPAS experiments}
Full results are shown in Table \ref{compas-table-appendix}.

\begin{table}
  \caption{Bias of COMPAS score (complete table)}
  \label{compas-table-appendix}
  \vskip 0.1in
  \centering
  \begin{tabular}{lllll}
    \toprule
    %\multicolumn{2}{c}{Part}                   \\
    %\cmidrule(r){1-2}
    type of bias & total & pos. & neg. & p-value \\
    \midrule
    $\bias_{\text{EO}}^S$ & 0.161 & 0\% &  100\% & <0.01\\ 
    $\bias_{\text{PE}}^S$ & 0.154 & 0\% &  100\% & <0.01 \\ 
    $\bias_{\text{CALI}}^S$ & 0.034  & 79\% & 21\% & 0.30\\ 
    \midrule
    $\bias_{\text{ROC}}$ & 0.016 & 46\%  & 54\% & 0.31 \\ 
    $\bias_{\text{xROC}}$ & 0.273 & 0\% & 100\% & <0.01 \\ 
    \midrule
    $\bias_{\text{EO}}^\mathcal{U}$ & 0.152  & 0\% & 100\% & <0.01\\ 
    $\bias_{\text{PE}}^\mathcal{U}$& 0.163 & 0\% & 100\% & <0.01\\ 
    $\bias_{\text{CALI}}^\mathcal{U}$&  0.037 &  78\%&  22\% & 0.23\\ 
    \bottomrule
  \end{tabular}
 \vskip -0.1in
\end{table}

\subsection{Details on German Credit data experiments}
Both models have been trained on 70\% of the dataset and evaluated on the remaining samples. We used min-max-scaling on continuous features and one-hot-encoding for categorical features.
Full results are shown in Table \ref{gcd-table-appendix}.
As the sample size is relatively small, it happens that even the large calibration biases are not statistically significant. 

\begin{table}
  \caption{Bias of models for German Credit data (complete table)}
  \label{gcd-table-appendix}
  \vskip 0.1in
  \centering
  \begin{tabular}{llllll}
    \toprule
    %\multicolumn{2}{c}{Part}                   \\
    %\cmidrule(r){1-2}
    type of bias & Model& total bias & pos. & neg. & p-value \\
    \midrule
    $\bias_{\text{EO}}^S$ & LogR & 0.083 & 1\% &  99\% & 0.04\\ 
     & LogR (debiased) & 0.048 & 93\% &  7\% & 0.32\\ 
     \midrule
    $\bias_{\text{PE}}^S$ & LogR & 0.092 & 0\% &  100\% & 0.09\\ 
     & LogR (debiased) &  0.025 & 62\% &  38\% & 0.99\\ 
     \midrule
     $\bias_{\text{CALI}}^S$ & LogR & 0.291 & 46\% &  54\% & 0.35\\ 
     & LogR (debiased) & 0.299 & 58\% &  42\% & 0.26\\ 
    \midrule
    \midrule
    $\bias_{\text{ROC}}$ & LogR & 0.044  & 98\%  & 2\% & 0.80 \\ 
     & LogR (debiased) & 0.050 & 98\% &  2\% & 0.69 \\ 
     \midrule
    $\bias_{\text{xROC}}$ & LogR & 0.133  & 0\% & 100\% & 0.02 \\ 
    & LogR (debiased) & 0.057 & 93\% & 7\% & 0.54\\ 
    \midrule
    \midrule
    $\bias_{\text{EO}}^\mathcal{U}$ & LogR & 0.041 & 3\% & 97\% & 0.13\\
    & LogR (debiased) & 0.036 & 97\% & 3\% & 0.23\\
     \midrule
    $\bias_{\text{PE}}^\mathcal{U}$&LogR &  0.078 &  1\% & 99\% & 0.10\\ 
    & LogR (debiased) &  0.024 & 74\% & 26\% & 0.98\\ 
     \midrule
    $\bias_{\text{CALI}}^\mathcal{U}$& LogR & 0.246  &  40\%&  60\% & 0.57\\ 
    & LogR (debiased) & 0.225 & 75\% & 25\% & 0.84\\ 
    \bottomrule
  \end{tabular}
  \vskip -0.1in
\end{table}

\subsection{Details on Adult experiments}
All three models have been trained on 70\% of the dataset and evaluated on the remaining samples. We removed the feature \emph{relationship}, which is highly entangled with \emph{sex} through the categories \emph{husband} and \emph{wife} and we engineered the remaining features to merge rare categories. We used min-max-scaling on continuous features and one-hot-encoding for categorical features.
Fig. \ref{adult-scores-logreg}-\ref{adult-scores-xgb} show the score distributions of the three scores on the testset. 
% Fig. \ref{adult-WEO-S} and \ref{adult-WEO-U} show the classifier bias over the score range and the quantile-transformed score range.

\begin{figure}
  \centering
  \includegraphics[width=0.5\columnwidth]{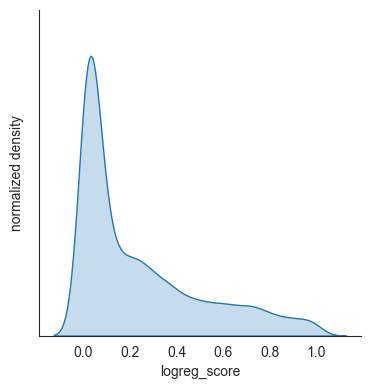}
  \caption{Distribution of logistic regression scores, trained on Adult data.}
  \label{adult-scores-logreg}
\end{figure}

\begin{figure}
  \centering
  \includegraphics[width=0.5\columnwidth]{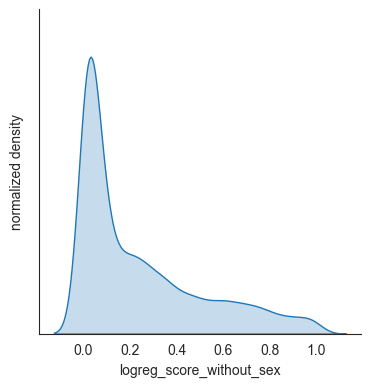}
  \caption{Distribution of logistic regression scores, trained on Adult data without protected attribute.}
  \label{adult-scores-logreg2}
\end{figure}

\begin{figure}
  \centering
  \includegraphics[width=0.5\columnwidth]{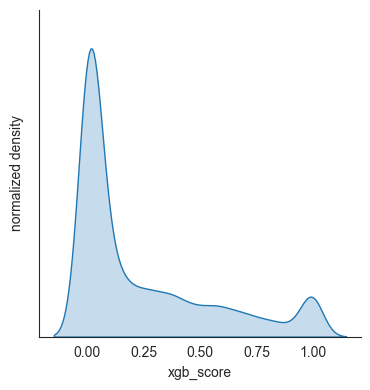}
  \caption{Distribution of XGBoost scores trained on Adult data.}
  \label{adult-scores-xgb}
\end{figure}

% \begin{figure}
%   \centering
%   \includegraphics[width=0.9\columnwidth]{figs/adult_W_EO_S_no_title.png}
%   \caption{Equal opportunity bias $\bias_{\text{EO}}^S$ of the logistic regression model trained on the Adult dataset. $\bias_{\text{EO}}^S$ is equal to the area under the curve of the true positive rate difference. The area is colored according to the group for which the bias part is favorable. }
%   \label{adult-WEO-S}
% \end{figure}

% \begin{figure}
%   \centering
%   \includegraphics[width=0.9\columnwidth]{figs/adult_W_EO_U_no_title.png}
%   \caption{Equal opportunity bias $\bias_{\text{EO}}^\mathcal{U}$ of the logistic regression model trained on the Adult dataset. $\bias_{\text{EO}}^\mathcal{U}$ is equal to the area under the curve of the true positive rate difference. The area is colored according to the group for which the bias part is favorable. }
%   \label{adult-WEO-U}
% \end{figure}


\end{document}